\PassOptionsToPackage{table}{xcolor}
\documentclass[10pt, a4paper, logo]{qwen}
\usepackage{wrapfig} 
\usepackage{times}
\usepackage{amsmath}
\usepackage{amssymb}
\usepackage{amsthm}
\usepackage{booktabs}
\usepackage{graphicx}
\usepackage{svg}
\usepackage{subcaption}
\usepackage{float}
\usepackage{multirow}
\usepackage{array}
\usepackage{tabularx}
\usepackage{xspace}
\usepackage{siunitx}
\usepackage[ruled,linesnumbered,lined]{algorithm2e}
\usepackage{enumitem}
\usepackage{natbib}
\usepackage{mdframed}

\definecolor{BestColor}{HTML}{C8E6C9}
\definecolor{FullTokenBlue}{HTML}{D7E6F7}
\definecolor{OursBlue}{HTML}{EAF3FF}
\definecolor{SkillPoolBorder}{HTML}{7561C9}
\definecolor{SkillPoolFill}{HTML}{F3F0FF}

\newcolumntype{C}[1]{>{\centering\arraybackslash}p{#1}}
\newcolumntype{L}[1]{>{\raggedright\arraybackslash}p{#1}}
\newcommand{\method}{\textsc{OmniDelta}\xspace}
\newcommand{\base}{\textsc{OmniZip}\xspace}

\newcommand{\paperfloatwidth}{\linewidth}
\newtheorem{proposition}{Proposition}
\newtheorem{corollary}{Corollary}
\newtheorem{assumption}{Assumption}
\newmdenv[
    linewidth=0.6pt,
    linecolor=black,
    backgroundcolor=black!2,
    skipabove=8pt,
    skipbelow=8pt,
    innerleftmargin=8pt,
    innerrightmargin=8pt,
    innertopmargin=7pt,
    innerbottommargin=7pt
]{promptbox}
\newmdenv[
    linewidth=0.8pt,
    linecolor=SkillPoolBorder,
    backgroundcolor=SkillPoolFill,
    skipabove=8pt,
    skipbelow=8pt,
    innerleftmargin=9pt,
    innerrightmargin=9pt,
    innertopmargin=8pt,
    innerbottommargin=8pt
]{skillbox}

\title{\textsc{OmniDelta}: Skill-Driven Budget Allocation for Token Compression in OmniLLMs}

\author[1,2]{Haoyang Huang\textsuperscript{*}}
\author[1,2]{Wenjie Huang\textsuperscript{*}}
\author[3,2]{Tianqi Xu\textsuperscript{*}}
\author[4,2]{Hongyaoxing Gu}
\author[1]{Kang Tan}
\author[1]{Yikai Fu}
\author[1,2]{Yuhao Shen}
\author[5]{Tianyu Liu}
\author[2]{Baolin Zhang}
\author[2]{Jun Zhang\textsuperscript{\ensuremath{\dagger,\ddagger}}}
\author[2]{Xinyi Hu\textsuperscript{\ensuremath{\dagger}}}
\author[2]{Jun Dai}
\author[2]{Shuang Ge\textsuperscript{\ensuremath{\ddagger}}}
\author[2]{Lei Chen}
\author[2]{Yue Li}
\author[2]{Mingchen Wang}
\author[1]{Meng Zhang\textsuperscript{\ensuremath{\dagger}}}
\affil[1]{Zhejiang University}
\affil[2]{Qwen Application, Alibaba}
\affil[3]{Carnegie Mellon University}
\affil[4]{University of Chinese Academy of Sciences}
\affil[5]{University of Science and Technology of China}

\begin{abstract}
Emerging Omni-modal Large Language Models (OmniLLMs) enable unified understanding of text, audio, and video, but their long audio-video token sequences introduce substantial memory and inference costs.
Existing compression methods mainly focus on selecting important tokens under fixed budgets, leaving the preceding budget-allocation problem underexplored.
We show that direct query-to-audio/video similarity is unreliable for inter-modal budget allocation, and that uniform intra-modal budgets can miss key evidence while retaining redundant content.
To address these limitations, we propose \method, a training-free \textbf{skill-driven} framework that couples \textbf{intent-aware} inter-modal allocation with \textbf{content-aware} intra-modal allocation.
\method first constructs audio and video skill pools to shift the fixed retained-token budget according to query demand, then reallocates modality budgets over audio segments and video frames using local complexity and temporal redundancy.
The resulting local budgets can be combined with existing pruning strategies, preserving the total retained-token ratio while changing where the budget is spent.
Experiments on four audio-video benchmarks with two \texttt{Qwen2.5-Omni} models show that \method establishes a new accuracy-efficiency Pareto frontier across pruning ratios. For example, at 25\% token retention on \texttt{Qwen2.5-Omni-7B}, \method reduces GPU memory by \textbf{22.0\%} and achieves a \textbf{1.64$\times$} end-to-end speedup over full-token inference.
\end{abstract}

\makeatletter
\renewcommand{\maketitle}{\bgroup\setlength{\parindent}{0pt}
    \begin{adjustwidth}{0pt}{24pt}
        \begin{flushleft}
            {
                {\raggedright \titlefont \@title\par}%
                \vskip11pt
                {\raggedright \@author\par}
                \vskip0pt {\raggedright \footnotesize \textsuperscript{*}Core contribution. \textsuperscript{\ensuremath{\dagger}}Corresponding author. \textsuperscript{\ensuremath{\ddagger}}Project leader.\par}
                \vskip10pt%
            }%
        \end{flushleft}
    \end{adjustwidth}
    \begin{center}
        \includegraphics[width=\textwidth]{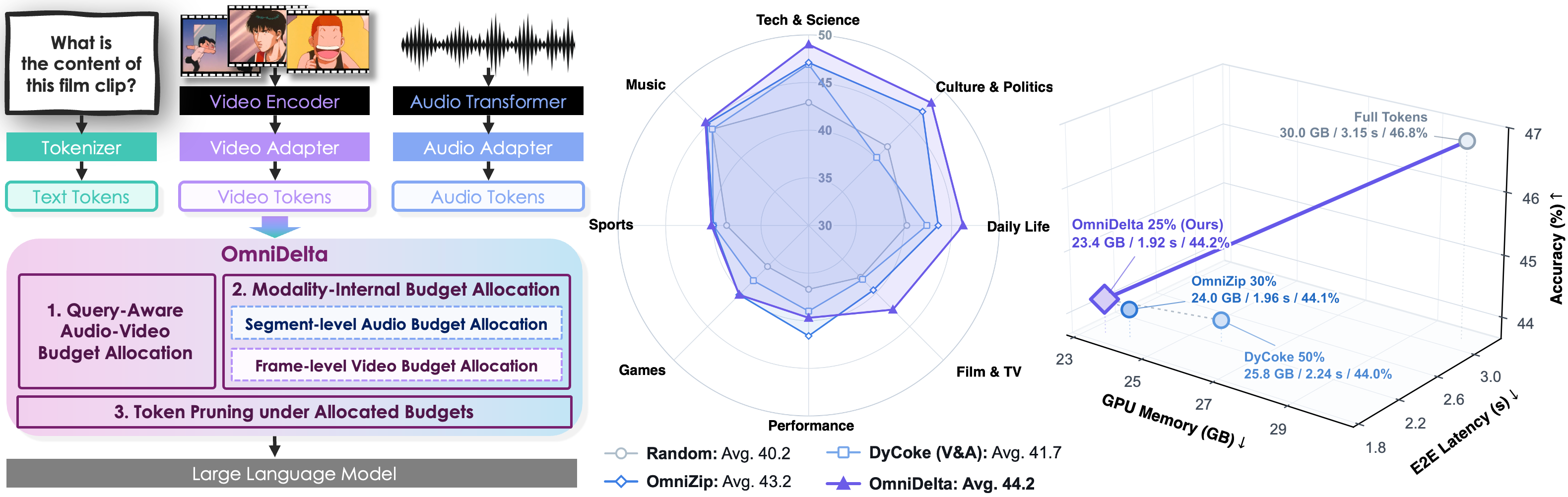}
        \captionof{figure}{Overview and main results of \method. Left: overview of \method. Middle: \method outperforms compressed baselines on WorldSense. Right: \method establishes a new Pareto frontier on Qwen2.5-Omni-7B by improving accuracy, reducing GPU memory, and lowering end-to-end latency.}
        \label{fig:teaser_overview}
    \end{center}
    \egroup
    {%
        {\abscontent}
    }%
    \thispagestyle{firststyle}
}%
\makeatother

\begin{document}
\maketitle
\section{Introduction}

Omni-modal large language models (OmniLLMs) extend multimodal reasoning from image/video-text inputs to unified understanding of text, audio, and video streams.
Recent systems such as VITA, VideoLLaMA, and Qwen-Omni demonstrate strong audio-video interaction capabilities by placing heterogeneous tokens into a shared autoregressive backbone~\citep{fu2024vita,cheng2024videollama2,qwen2025qwen25omni,qwen2025qwen3omni}.
However, this unified interface also makes inference expensive.
For example, in Qwen-Omni models, a 60-second video sampled at 2 FPS contains 120 frames; with 72 visual tokens per frame and 25 audio tokens per second, the audio-video input alone reaches roughly 10K tokens before adding the text query and special tokens.
Such long multimodal contexts substantially increase memory usage and prefilling cost, making token compression essential for practical OmniLLM inference~\citep{zhang-etal-2026-efficient}.

Existing compression methods mainly ask which tokens should be retained.
Image and video pruning methods reduce sequence length by removing redundant visual tokens~\citep{bolya2022tome,chen2024fastv,tao2025dycoke,fu2025framefusion}.
In contrast, token compression in omni-modal models must account for cross-modal relevance between audio and video when identifying redundant tokens.
\base uses audio attention to guide video pruning in temporally aligned windows~\citep{tao2025omnizip};
OmniSIFT and EchoingPixels learn cross-modal token selectors with training-based straight-through estimators~\citep{ding2026omnisift,gong2025echoingpixels};
OmniRefine and OmniFit refine chunk-wise or layer-wise compression policies~\citep{deng2026omnirefine,wang2026omnifit}.
Despite their different token selection mechanisms, these methods operate under predefined budgets: modality-level and intra-modal budgets remain fixed, while OmniFit follows a fixed layer-wise budget schedule.
They therefore optimize token selection without investigating how the retained-token budget should adapt to the query and input content.
This raises a complementary question: \textit{\textbf{Before selecting which tokens to retain, how should the fixed budget be allocated?}}

This budget-allocation problem is especially important for OmniLLMs because the usefulness of audio and video depends on both the query and the sample.
A sound-centric question may require more audio tokens, while a visual-detail question may require more video tokens.
At the same time, information is unevenly distributed within each modality.
For example, a video contains information-rich key frames as well as low-information transitional frames, while audio likewise alternates between informative and redundant moments.
Assigning the same budget to all frames or moments may waste tokens on uninformative content while leaving insufficient capacity to preserve critical evidence.

Therefore, we investigate two questions. \textbf{Question 1:} \textit{\textbf{How can the relevance of a query to the audio and video modalities be measured?}} Direct query-to-audio/video cosine similarity appears natural, but our analysis finds it unreliable for both encoder-projector embeddings and thinker hidden states: the query primarily expresses task intent, whereas modality tokens encode sample-specific content. \textbf{Question 2:} \textit{\textbf{Does uniform intra-modal budget allocation hinder the preservation of critical evidence?}} In our audio and video probes, GPT-5.5-xhigh is given the reference answer to identify the most informative content for retention. Even with this strong pruning strategy, a uniform budget performs substantially worse than allowing the model to determine both the budget distribution and retained content under the same total budget. This indicates that token pruning alone cannot compensate for a suboptimal budget allocation.

Motivated by these observations, we introduce \method~\footnote{``Omni'' denotes omni-modal inputs, while ``Delta'' evokes a river delta that redistributes a conserved flow, mirroring fixed-budget reallocation between and within modalities.}, a training-free hierarchical \textbf{budget allocator} for OmniLLM token compression.
As illustrated in Fig.~\ref{fig:teaser_overview}, inspired by agentic skill methods that organize reusable capabilities through textual descriptions~\citep{jiang2026agenticskills}, we construct \textbf{modality-specific skill pools} by selecting representative keywords for audio- and video-oriented task types and expanding them with semantically related terms.
\method shifts the inter-modal budget according to the cosine similarity between the query and skill terms after both are mapped through the model's text embedding layer.
Within each modality, the resulting budget is further redistributed across audio segments or video frames according to local complexity and temporal redundancy.
Finally, tokens are pruned under the allocated local budgets using audio attention and video spatio-temporal redundancy, while remaining compatible with existing token-pruning strategies.
Thus, \method changes where the budget is spent without changing the total retained-token ratio.

Our contributions are summarized as follows:
\begin{itemize}[leftmargin=*]
    \item \textbf{Budget-allocation diagnostics.} We identify budget allocation as a distinct problem in OmniLLM token compression, complementary to token pruning, and show that direct query-to-audio/video similarity for inter-modal allocation and uniform intra-modal budgets are unreliable principles.
    \item \textbf{Skill-driven budget allocation.} We propose \method, a training-free hierarchical allocator that redistributes a fixed retained-token budget across audio and video modalities, audio segments, and video frames using query-skill routing, local complexity, and temporal redundancy.
    \item \textbf{Accuracy--efficiency gains.} We evaluate \method on four audio-video benchmarks and two \texttt{Qwen2.5-Omni} model sizes under multiple compression settings, showing that \method consistently advances the accuracy-efficiency Pareto frontier. At 25\% token retention on \texttt{Qwen2.5-Omni-7B}, \method reduces GPU memory by 22.0\% and achieves a 1.64$\times$ end-to-end speedup over full-token inference.
\end{itemize}

\section{Motivation}
\label{sec:motivation}

Most OmniLLM compression studies emphasize which tokens should be removed, while the retained budget is often fixed across queries or assigned by a predefined temporal policy.
Before designing a query-aware budget allocator, we first examine two assumptions behind budget assignment.
At the modality level, direct text-to-audio/video similarity is a natural routing signal, but it may not reflect the modality required by the query.
At the temporal level, window-based policies such as \base preserve audio-video alignment, yet their shared local budgets may still miss visually critical regions when evidence is uneven~\citep{tao2025omnizip}.
This section studies these two issues through controlled diagnostics.

\begin{figure}[!t]
    \centering
    \includegraphics[width=\linewidth]{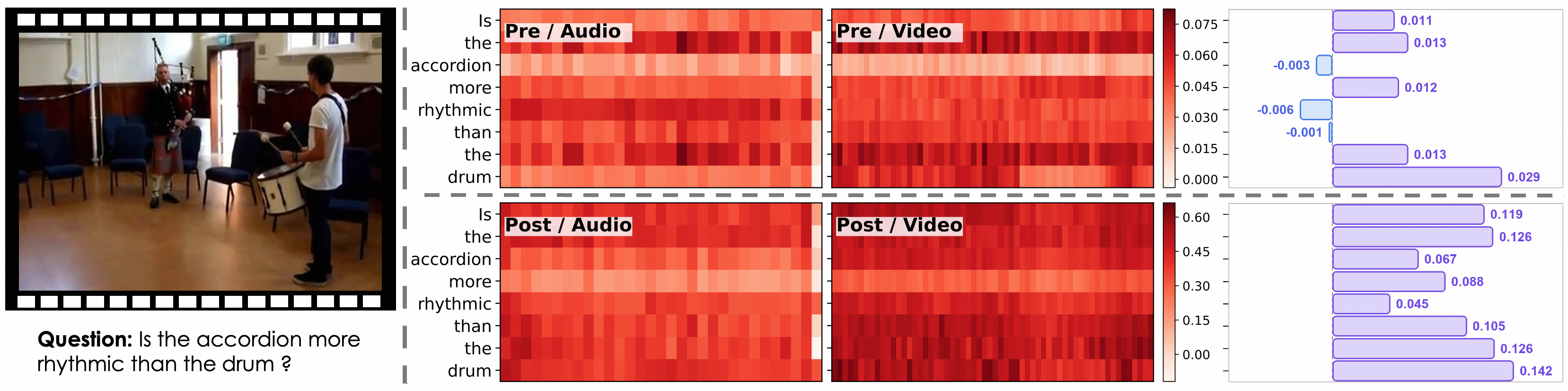}
    \caption{Representative failure of direct query-to-audio/video Top-3 similarity. The left shows an audio-oriented question, the center shows word-to-audio/video token similarities before and after the Thinker, and the right reports the aligned Top-3 margins $\Delta=s_{\mathrm{v}}-s_{\mathrm{a}}$, where positive values favor video.}
    \label{fig:cross_modal_similarity_case}
\end{figure}

\subsection{Unreliability of Direct Query-to-Audio/Video Similarity}
\label{sec:similarity_analysis}

At the modality-allocation level, a query-aware compressor must determine whether a question requires more evidence from audio or video.
A straightforward strategy is to compare the query representation against the audio and video token representations and allocate a larger budget to the modality with higher similarity.
However, this strategy can fail because the query primarily encodes task intent, whereas the modality tokens capture sample-specific content.
This observation motivates our first diagnostic question:

\noindent\textbf{Question 1:} \textit{Can direct query-to-audio/video similarity reliably guide modality-level budget allocation?}

We evaluate whether direct query-to-audio/video similarity can serve as a reliable modality router.
We use GPT-5.5-xhigh to categorize WorldSense queries~\citep{hong2026worldsense} according to the evidence modality required to answer them, and sample 200 audio-oriented and 200 video-oriented queries to construct a balanced diagnostic set of 400 queries.
For each query, the same model selects one \emph{modality keyword}, defined as the query word most indicative of whether acoustic or visual evidence is required.
We use this automatically selected keyword as a controlled diagnostic view rather than as a ground-truth semantic annotation.

All probes process the complete audio-video sequence without pruning.
Mean pooling averages similarities over all modality tokens and therefore measures global alignment.
In contrast, top-3 mean pooling averages the three largest similarities, testing whether a useful routing signal is localized to a small subset of highly matched tokens and would otherwise be diluted by full-sequence averaging.
We fix $k=3$ across all settings to preserve localized evidence while reducing sensitivity to a single spuriously high match.
We evaluate both the whole query and the modality keyword under these two pooling rules at three representation stages: pre-Thinker (input-query embeddings versus encoder-projector modality tokens), mid-Thinker (layer 14), and post-Thinker (the final hidden states before the LM head).

\begin{wrapfigure}{r}{0.4\textwidth}
    \centering
    \includegraphics[width=\linewidth]{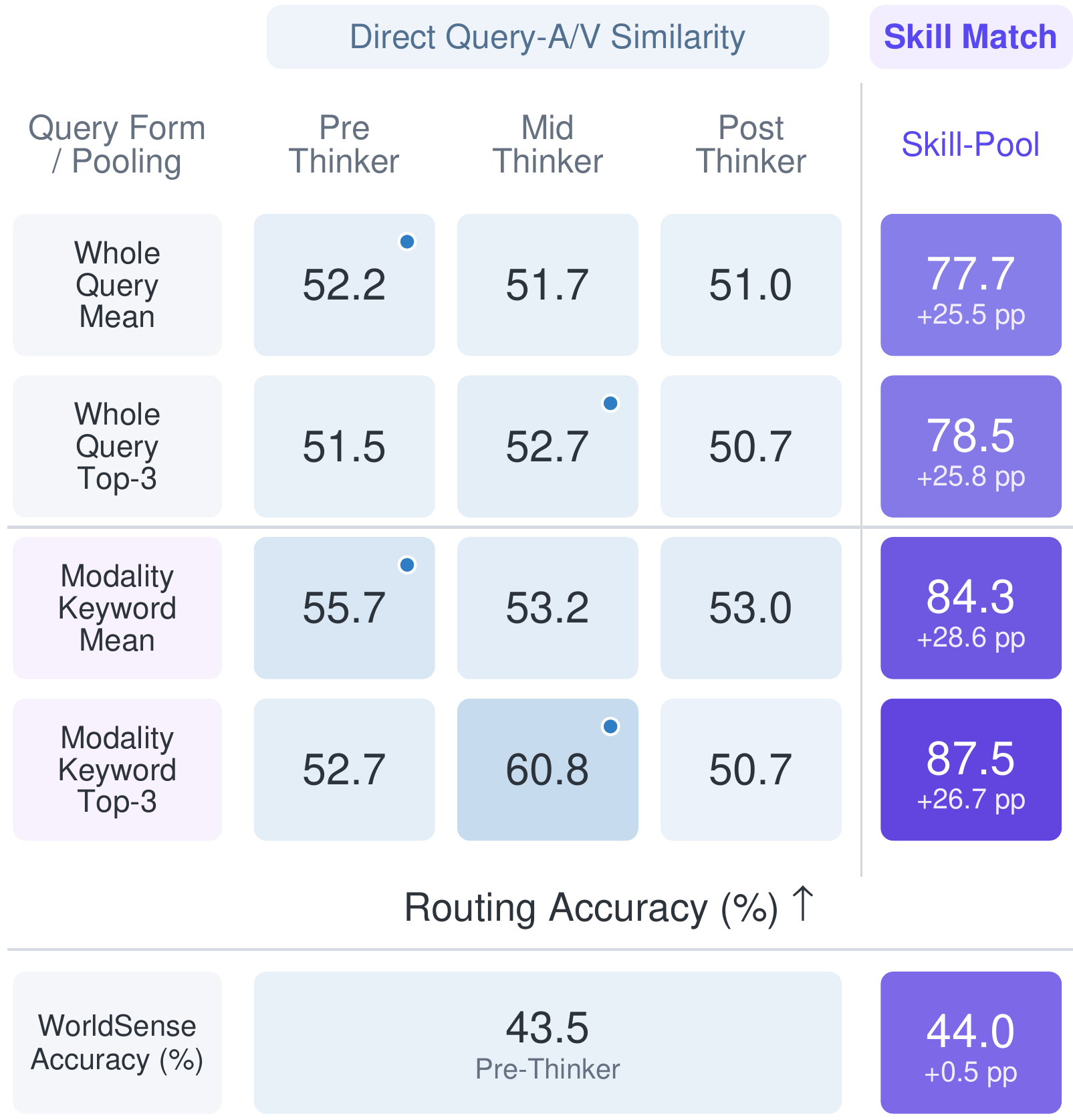}
    \caption{Routing accuracy on 400 WorldSense queries. Blue dots mark the maximum among the three direct-similarity columns in each row, and the rightmost skill-pool cell reports its improvement over that value; the bottom row shows WorldSense accuracy.}
    \label{fig:router_similarity_comparison}
    \vspace{-3em}
\end{wrapfigure}

Let $\mathbf{E}_{q}$ be the input query embeddings, $\mathbf{H}^{(\ell)}$ be the hidden states after the $\ell$-th Thinker block, and $\mathbf{H}^{(\mathrm{F})}$ be the final hidden states before the LM head.
Let $\mathbf{A}$ and $\mathbf{V}$ denote the encoder-projector audio and video token sets, respectively.
For a modality keyword $w$ with token positions $\mathcal{P}_w$, its stage-specific representation is
\begin{equation*}
    \mathbf{q}_{w}^{(0)}
    =
    \frac{1}{|\mathcal{P}_w|}
    \sum_{i\in\mathcal{P}_w}\mathbf{E}_{q,i},
    \qquad
    \mathbf{q}_{w}^{(\ell)}
    =
    \frac{1}{|\mathcal{P}_w|}
    \sum_{i\in\mathcal{P}_w}\mathbf{H}_{i}^{(\ell)},
    \quad \ell\in\{14,\mathrm{F}\}.
\end{equation*}
For whole-query routing, we replace $\mathcal{P}_w$ with $\mathcal{P}_q$, the set of all query-token positions.
The modality token sets are $\mathbf{Z}_{\mathrm{a}}^{(0)}=\mathbf{A}$ and $\mathbf{Z}_{\mathrm{v}}^{(0)}=\mathbf{V}$ at the pre-Thinker stage, and $\mathbf{Z}_{m}^{(\ell)}=\{\mathbf{z}_{m,j}^{(\ell)}\mid j\in\mathcal{I}_m\}$ at the Thinker stages, where $\mathcal{I}_m$ indexes the tokens of modality $m$ and $\mathbf{z}_{m,j}^{(\ell)}$ denotes an individual modality-token representation.
Given a pooling operator $g\in\{\operatorname{Mean},\operatorname{TopKMean}\}$, with $k=3$ for TopKMean, the query-to-modality similarity and the corresponding modality margin are
\begin{equation*}
    s_{w,m,g}^{(\ell)}
    =
    g_{j\in\mathcal{I}_m}
    \left(
    \frac{(\mathbf{q}_{w}^{(\ell)})^\top\mathbf{z}_{m,j}^{(\ell)}}
    {\|\mathbf{q}_{w}^{(\ell)}\|_2\|\mathbf{z}_{m,j}^{(\ell)}\|_2}
    \right),
    \qquad
    \Delta_{w,g}^{(\ell)}
    =
    s_{w,\mathrm{v},g}^{(\ell)}
    -
    s_{w,\mathrm{a},g}^{(\ell)}.
\end{equation*}
A positive margin predicts a video-oriented query, whereas a negative margin predicts an audio-oriented query.

Fig.~\ref{fig:cross_modal_similarity_case} illustrates a representative failure case for an audio-oriented query under Top-3 pooling.
At the pre-Thinker stage, ``accordion'' and the selected modality keyword ``rhythmic'' are only weakly audio-oriented, while ``drum'' and most other words already favor video; at the post-Thinker stage, every word has a positive video-minus-audio margin, including these acoustically salient terms.
The aligned heatmaps and margins show that deeper contextualization increases cross-modal mixing without recovering the required audio modality, so direct similarity can reflect sample content rather than the evidence needed to solve the task.

As summarized in Fig.~\ref{fig:router_similarity_comparison}, direct cosine routing remains close to random guessing across representation stages and pooling rules, with no consistent benefit from deeper representations or selective pooling; even the best direct variant reaches only 60.8\%.
Skill-pool routing is substantially stronger for both whole-query and keyword inputs, outperforming the matched direct variants by 25.5--28.6 percentage points.
This confirms that task prototypes better reflect query intent.
This diagnostic advantage translates into a smaller end-to-end gain.
Under the same \base intra-modal allocator and pruning backend, the embedding-similarity and skill-pool routers obtain 43.5\% and 44.0\% WorldSense accuracy, respectively.
Routing accuracy is therefore mechanism-level evidence; final performance also depends on intra-modal allocation and pruning.
Appendix~\ref{app:routing_details} provides the complete protocol and results, and Sec.~\ref{sec:skill_budget} introduces the resulting skill-based allocator.

\subsection{Suboptimality of Uniform Intra-Modal Budget Allocation}
\label{sec:fine_grained_budget_motivation}

\begin{figure}[!t]
    \vspace{-1em}
    \centering
    \includegraphics[width=\linewidth]{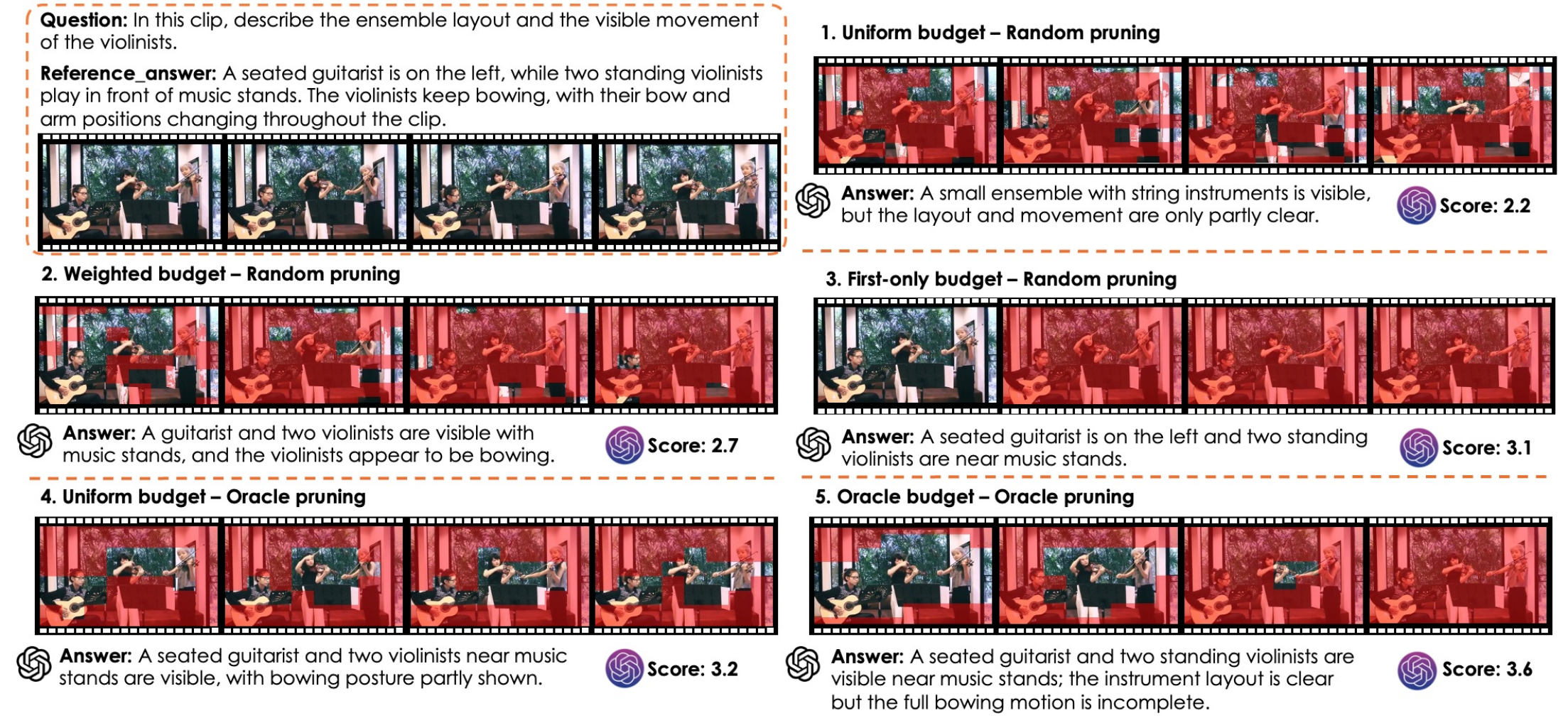}
    \caption{Fine-grained budget allocation diagnostic. Given the same local clip, different budget distributions and pruning policies expose different visible evidence to the answering model.}
    \label{fig:fine_grained_budget_allocation}
    \vspace{-1em}
\end{figure}

After the top-level modality split, the remaining question is how finely each modality should spend its own budget.
Window-level policies preserve temporal correspondence, but identical budgets inside a local window may retain redundant views while under-preserving key evidence.
This motivates the second key question:

\noindent\textbf{Question 2:} \textit{Does assigning identical budgets within a local window preserve redundant content while leaving insufficient capacity to capture key evidence?}

We test this with parallel diagnostics on video and audio.
We use \emph{local unit} to denote the element receiving an individual budget, namely one video frame or one audio segment.
For video, we sample 50 2-second WorldSense clips, extract four evenly spaced frames from each clip, and keep 25\% visible content.
GPT-5.5-xhigh generates a detail- and motion-oriented question with a reference answer from the full clip, answers using only the retained content, and scores the answer against the reference.
For audio, we sample 50 real WorldSense audio clips, take 12-second excerpts, split each excerpt into four 3-second segments, and prune at the granularity of 0.5-second chunks.
All audio policies keep 12 out of 24 chunks, with removed chunks replaced by silence; Gemini-3.1-Pro-preview follows the same generate-answer-score protocol.

\begin{wrapfigure}{r}{0.5\textwidth}
    \centering
    \includegraphics[width=\linewidth]{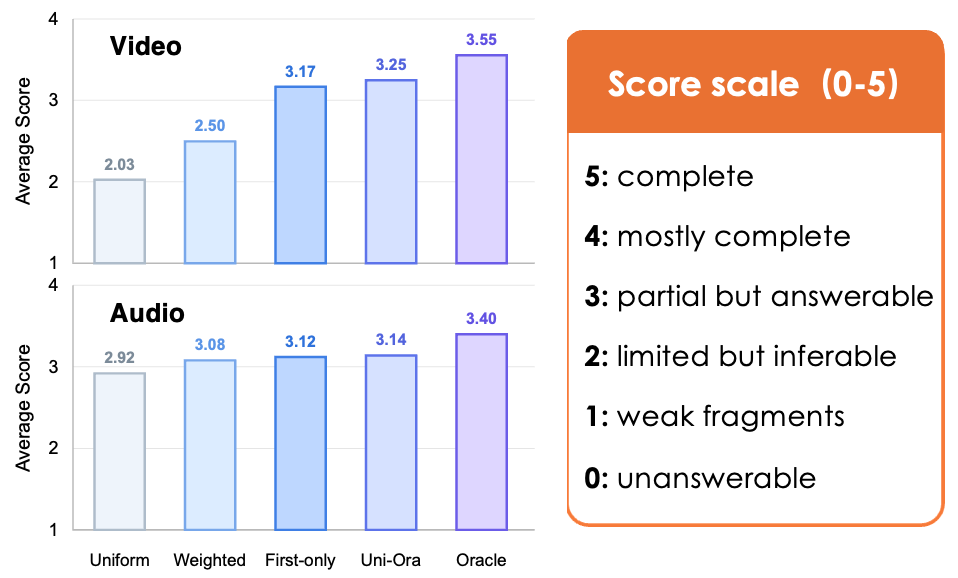}
    \caption{Average answer scores for video (top) and audio (bottom) budget diagnostics under five budget/pruning policies, together with the 0--5 score scale.}
    \label{fig:fine_grained_budget_scores}
\end{wrapfigure}

Fig.~\ref{fig:fine_grained_budget_allocation} and~\ref{fig:fine_grained_budget_scores} compare five policies.
\emph{Uniform budget--Random pruning} assigns equal budget to each unit and randomly keeps content within it.
\emph{Weighted budget--Random pruning} uses a fixed front-heavy split across the four units, $[0.70,0.15,0.10,0.05]$ for video and $[5,4,2,1]$ chunks for audio.
\emph{First-only} keeps only the earliest units, i.e., the first video frame or the first two audio segments.
\emph{Uniform budget--Oracle pruning} keeps equal unit budgets but uses answer-aware content selection, while \emph{Oracle budget--Oracle pruning} chooses both the unit budgets and the retained content under the same total budget.

Fig.~\ref{fig:fine_grained_budget_scores} shows consistent trends across audio and video.
Even with an answer-aware oracle pruning policy, Uniform--Oracle remains far below Oracle--Oracle because its budget allocation is still uniform, and its performance is only close to First-only.
This gap shows that pruning alone cannot compensate for an unsuitable budget split: even if the selector keeps the most important tokens within each unit, insufficient budget on informative units may still discard key evidence, while redundant units continue to occupy retained tokens.
Therefore, budget allocation is a critical step before token selection, motivating a finer allocator that redistributes the fixed retained budget across local units according to redundancy and information density.
The resulting intra-modal budget reallocation is detailed in Sec.~\ref{sec:audio_budget}.
Complete prompts and experimental details for the video and audio diagnostics are provided in Appendix~\ref{app:fine_grained_details}.

\section{Method}
\label{sec:method}

\subsection{Background on OmniLLMs}
\label{sec:background}

Omni-modal large language models (OmniLLMs) take synchronized video, audio, and text as a unified input sequence.
Given a video clip $\mathcal{V}$ and its audio waveform $\mathcal{A}$, we use $\Phi_{\mathrm{v}}(\cdot)$ and $\Phi_{\mathrm{a}}(\cdot)$ to denote the visual and audio encoder-projector pipelines, respectively, which map raw modality inputs into the LLM hidden space.
The resulting visual and audio tokens are denoted as
\begin{equation}
    \mathbf{V}=\Phi_{\mathrm{v}}(\mathcal{V})\in\mathbb{R}^{N_{\mathrm{v}}\times d},
    \qquad
    \mathbf{A}=\Phi_{\mathrm{a}}(\mathcal{A})\in\mathbb{R}^{N_{\mathrm{a}}\times d},
    \label{eq:omni_tokens}
\end{equation}
where $N_{\mathrm{v}}$ and $N_{\mathrm{a}}$ are the numbers of visual and audio tokens, and $d$ is the LLM hidden dimension.
The user instruction is tokenized as a text sequence $\mathbf{T}$, and the LLM receives an interleaved multimodal sequence composed of text, audio, and video tokens.

To preserve local audio-video correspondence, OmniLLMs usually organize projected audio and video tokens into temporally aligned chunks.
Let $\mathcal{C}_j=[\mathbf{V}_j;\mathbf{A}_j]$ denote the $j$-th multimodal chunk, where $\mathbf{V}_j$ and $\mathbf{A}_j$ are the visual and audio tokens belonging to the same time span.
The final input can be written as
\begin{equation}
    \mathbf{X}=[\mathbf{T};\mathcal{C}_1;\ldots;\mathcal{C}_J],
    \label{eq:omni_sequence}
\end{equation}
which is processed by the LLM backbone during prefilling and generation.
Since self-attention cost grows rapidly with sequence length, reducing audio-visual tokens before they enter the backbone is an effective way to improve inference efficiency.
However, audio and video have different redundancy patterns and their usefulness depends on the query.
This makes omni-modal compression not only a token-ranking problem, but also a budget allocation problem: under a fixed retained-token budget, the model must decide how much budget to assign to each modality and to each temporal region.

\begin{figure}[!t]
    \centering
    \includegraphics[width=\linewidth]{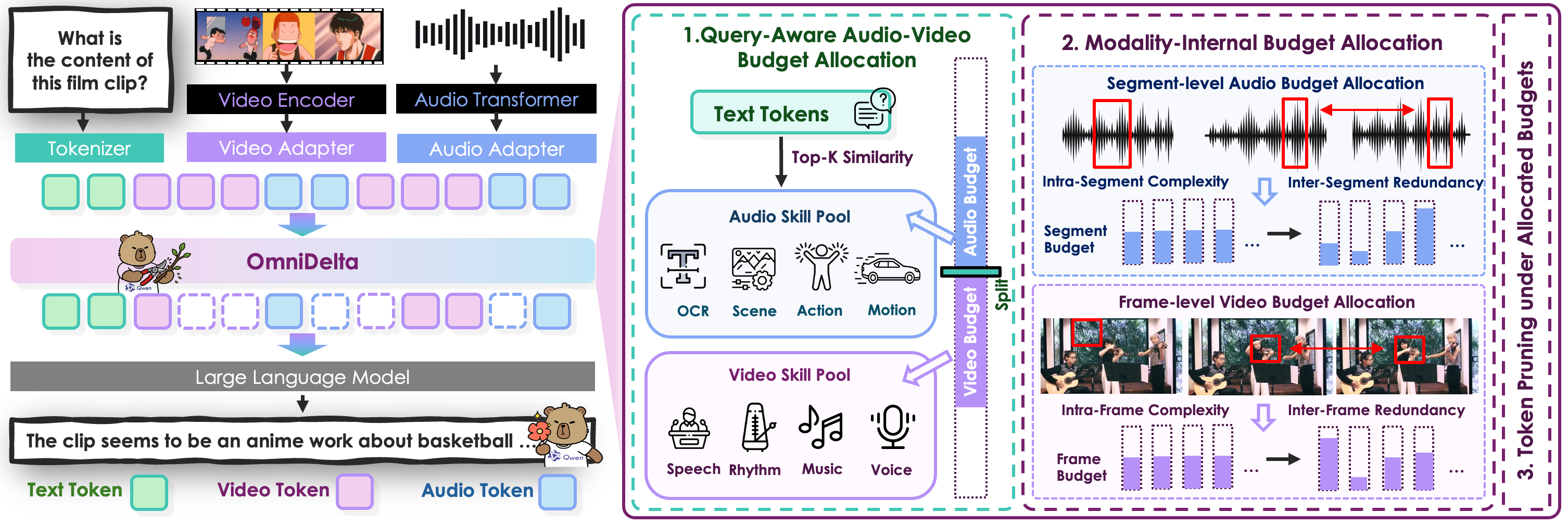}
    \caption{Overview of \method. The method keeps the standard OmniLLM input pipeline and inserts a budget-allocation layer before token selection. It first shifts budget between audio and video using query-skill similarity, then redistributes modality budgets over video frames and audio segments according to local complexity and temporal redundancy.}
    \label{fig:method_overview}
\end{figure}

\subsection{Our Method: \method}
\label{sec:omnidelata}

As shown in Fig.~\ref{fig:method_overview}, \method formulates omni-modal token compression as a hierarchical budget allocation problem.
Given $N_{\mathrm{a}}$ audio tokens and $N_{\mathrm{v}}$ video tokens, we first fix the global retained budget as
\begin{equation}
    K=\operatorname{round}\left(r(N_{\mathrm{a}}+N_{\mathrm{v}})\right),
    \label{eq:total_budget}
\end{equation}
where $r$ is the target retained ratio.
\method decomposes token compression into hierarchical budget allocation followed by token pruning.
It first transfers budget between audio and video according to query semantics, and then redistributes each modality budget across audio segments or video frames.
Given these local budgets, the pruning stage selects the concrete tokens retained within each unit.
We initialize the modality budgets from a fixed compression prior, denoted by $(K_{\mathrm{a}}^0,K_{\mathrm{v}}^0)$ with $K_{\mathrm{a}}^0+K_{\mathrm{v}}^0=K$.
This prior gives a stable starting point, while \method only changes where the fixed budget is spent.

\subsubsection{Intent-Aware Inter-Modal Budget Allocation}
\label{sec:skill_budget}

The top level determines whether the query requires more audio or video evidence.
As shown in Sec.~\ref{sec:similarity_analysis}, direct query-to-audio/video token similarity is not a reliable modality router.
We therefore construct modality skill pools offline with GPT-5.5-xhigh.
From WorldSense, it identifies representative audio-oriented task types, including speech and dialogue, speaker and voice, music and rhythm, and sound-event understanding, as well as video-oriented task types, including object and scene recognition, appearance and spatial relations, OCR, and action and motion understanding.
It extracts diagnostic keywords for these tasks and expands them with semantically related terms.
We denote the resulting audio and video skill pools by $\mathcal{S}_{\mathrm{a}}$ and $\mathcal{S}_{\mathrm{v}}$, respectively; construction details and representative entries are provided in Appendix~\ref{app:skill_pool}.
At inference time, the query and skill entries are compared in the OmniLLM input-embedding space.
Let $\mathbf{E}(\cdot)$ denote the thinker input embedding layer.
For a text string $x$ with token positions $\mathcal{P}_x$, we define its normalized pooled embedding as
\begin{equation}
    \mathbf{e}(x)
    =
    \operatorname{Norm}
    \left(
    \frac{1}{|\mathcal{P}_x|}
    \sum_{i\in\mathcal{P}_x}\mathbf{E}(x_i)
    \right),
    \qquad
    r_m(q)
    =
    \frac{1}{k}
    \sum_{c\in\operatorname{TopK}_k(q,\mathcal{S}_m)}
    \mathbf{e}(q)^{\top}\mathbf{e}(c),
    \quad
    m\in\{\mathrm{a},\mathrm{v}\}.
    \label{eq:skill_score}
\end{equation}
Here $\operatorname{Norm}(\cdot)$ denotes $\ell_2$ normalization, so the inner product equals cosine similarity.
$\operatorname{TopK}_k(q,\mathcal{S}_m)$ selects the $k$ skills in modality $m$ with the highest similarity to the query.
The two modality scores are normalized into probabilities:
\begin{equation}
    [p_{\mathrm{a}},p_{\mathrm{v}}]
    =
    \operatorname{Softmax}
    \left(
    [r_{\mathrm{a}}(q),r_{\mathrm{v}}(q)]
    \right),
    \qquad
    m_q=p_{\mathrm{v}}-p_{\mathrm{a}}.
    \label{eq:skill_probability}
\end{equation}
Here $m_q$ is the signed modality bias.
A positive $m_q$ indicates stronger video demand, while a negative $m_q$ indicates stronger audio demand. We use this signed bias to shift budget from the less relevant modality to the more relevant one:
\begin{equation}
    \Delta_q=\lambda_q K m_q,
    \qquad
    K_{\mathrm{a}}=K_{\mathrm{a}}^0-\Delta_q,
    \qquad
    K_{\mathrm{v}}=K_{\mathrm{v}}^0+\Delta_q.
    \label{eq:modality_budget}
\end{equation}
Here $K$ is the total retained budget, $\lambda_q$ controls the maximum query-driven transfer strength, and $(K_{\mathrm{a}}^0,K_{\mathrm{v}}^0)$ denotes the prior audio/video budget split.
Thus, video-oriented queries move budget from audio to video, audio-oriented queries move budget in the opposite direction, and ambiguous queries produce a small shift.
After range projection and integer correction, the modality budgets still satisfy $K_{\mathrm{a}}+K_{\mathrm{v}}=K$.

\subsubsection{Content-Aware Intra-Modal Budget Allocation}
\label{sec:audio_budget}
\label{sec:video_budget}

After the top-level split, \method performs a finer-grained budget allocation inside each modality.
For audio, the minimum allocation unit is an audio segment, where we group the audio tokens within one second into a segment.
For video, the minimum allocation unit is a single frame.
Since audio segments and video frames follow the same allocation principle, we use a unified notation and omit the modality superscript when there is no ambiguity.
For a modality $m\in\{\mathrm{a},\mathrm{v}\}$, let $u_i$ denote the $i$-th local unit, containing $L_i$ token features $\{\mathbf{x}_{i,j}\}_{j=1}^{L_i}$.
Its mean representation and intra-unit complexity are defined as
\begin{equation}
    \boldsymbol{\mu}_i
    =
    \frac{1}{L_i}
    \sum_{j=1}^{L_i}\mathbf{x}_{i,j},
    \qquad
    C_i^m
    =
    \operatorname{Norm}
    \left(
    1-\frac{1}{L_i}\sum_{j=1}^{L_i}
    \cos(\mathbf{x}_{i,j},\boldsymbol{\mu}_i)
    \right).
    \label{eq:unit_complexity}
\end{equation}
A large $C_i^m$ means that tokens inside the unit are diverse, so the unit should retain more tokens.
We also compute inter-unit redundancy against the previous temporal unit.
Instead of comparing unit means, this score averages cosine similarities between corresponding token positions:
\begin{equation}
    R_i^m
    =
    \begin{cases}
    0, & i=1,\\
    \operatorname{Norm}_{t>1}
    \left(
    \frac{1}{\tilde{L}_i}
    \sum_{j=1}^{\tilde{L}_i}
    \cos\left(\tilde{\mathbf{x}}_{i,j},\tilde{\mathbf{x}}_{i-1,j}\right)
    \right), & i>1.
    \end{cases}
    \label{eq:unit_redundancy}
\end{equation}
Here $\tilde{\mathbf{x}}_{i,j}$ and $\tilde{\mathbf{x}}_{i-1,j}$ are aligned token features after resampling adjacent units to a common length $\tilde{L}_i$, and $\operatorname{Norm}_{t>1}(\cdot)$ denotes min-max normalization over units with a previous neighbor.
A large $R_i^m$ means that the current unit repeats token-level patterns from the previous unit and can be compressed more aggressively.
We combine these two signals into a centered signed score:
\begin{equation}
    z_i^m
    =
    (R_i^m-C_i^m)
    -
    \frac{1}{n_m}\sum_{j=1}^{n_m}(R_j^m-C_j^m)
    .
    \label{eq:unit_score}
\end{equation}
If $z_i^m>0$, the unit is relatively more redundant; if $z_i^m<0$, the unit is relatively more complex.
Let $k_i^{0,m}$ be the prior keep count of unit $u_i$, with $\sum_i k_i^{0,m}=K_m$.
Let $\lambda_m$ be the modality-specific shift coefficient, where $\lambda_{\mathrm{a}}$ and $\lambda_{\mathrm{v}}$ control the strength of local budget adjustment for audio segments and video frames, respectively.
\method adjusts these prior counts by constructing a feasible keep interval for each unit:
\begin{equation}
    \ell_i^m
    =
    k_i^{0,m}-\lambda_m[z_i^m]_+L_i,
    \qquad
    h_i^m
    =
    k_i^{0,m}+\lambda_m[-z_i^m]_+L_i.
    \label{eq:unit_keep_bounds}
\end{equation}
Here $[x]_+=\max(x,0)$.
Thus, redundant units are allowed to reduce their keep counts, while complex units are allowed to increase their keep counts.
Starting from the lower bounds, we distribute the remaining budget according to a keep priority:
\begin{equation}
    \begin{aligned}
    w_i^m&=\frac{1-z_i^m}{2},\\
    k_i^m
    &=
    \ell_i^m+
    \left(K_m-\sum_j\ell_j^m\right)
    \frac{w_i^m(h_i^m-\ell_i^m)}
    {\sum_j w_j^m(h_j^m-\ell_j^m)}.
    \end{aligned}
    \label{eq:unit_keep_allocation}
\end{equation}
The same rule applies to audio segments and video frames.
It assigns more budget to units with high complexity and low redundancy, while preserving $\sum_i k_i^m=K_m$ after rounding and boundary correction.

\begin{table}[!t]
\centering
\caption{WorldSense category results under 25\% and 20\% audio-visual retained-token settings. Best compressed results, excluding Full Tokens, are bolded.}
\label{tab:worldsense_main}
\scriptsize
\setlength{\tabcolsep}{2.5pt}
\resizebox{\paperfloatwidth}{!}{%
\begin{tabular}{L{2.60cm}C{0.55cm}C{0.50cm}*{8}{C{0.82cm}}C{0.50cm}C{0.75cm}}
\toprule
\textbf{Method} & \textbf{Ret.} & & \textbf{Tech} & \textbf{Cult.} & \textbf{Daily} & \textbf{Film} & \textbf{Perf.} & \textbf{Games} & \textbf{Sports} & \textbf{Music} & & \textbf{Avg.} \\
\midrule
\multicolumn{13}{c}{\emph{Qwen2.5-Omni-7B}} \\
\midrule
\rowcolor{FullTokenBlue}
Full Tokens & 100\% & & 52.0 & 51.1 & 48.0 & 44.6 & 43.4 & 42.1 & 42.1 & 47.5 & & 46.8 \\
Random & 25\% & & 42.9 & 41.7 & 40.3 & 37.7 & 36.7 & 36.1 & 38.6 & 44.3 & & 40.2 \\
DyCoke (V\&A) & 25\% & & 46.9 & 40.1 & 42.4 & 38.0 & 39.0 & 38.2 & 40.0 & 44.3 & & 41.7 \\
\base & 25\% & & 47.1 & 46.9 & 43.6 & 39.6 & 40.4 & 39.9 & 40.0 & 45.1 & & 43.2 \\
\base & 20\% & & 46.9 & 46.0 & 41.8 & 40.1 & \textbf{41.6} & \textbf{40.3} & 40.0 & 44.1 & & 42.7 \\
\rowcolor{OursBlue}
\mbox{\method (Ours)} & 25\% & & \textbf{49.0} & \textbf{48.2} & \textbf{46.2} & \textbf{42.5} & 39.7 & 39.1 & 38.8 & \textbf{45.3} & & \textbf{44.2} \\
\rowcolor{OursBlue}
\mbox{\method (Ours)} & 20\% & & 47.1 & 46.3 & 42.9 & 40.9 & 39.3 & \textbf{40.3} & \textbf{40.2} & 44.3 & & 43.0 \\
\midrule
\multicolumn{13}{c}{\emph{Qwen2.5-Omni-3B}} \\
\midrule
\rowcolor{FullTokenBlue}
Full Tokens & 100\% & & 51.6 & 51.8 & 44.5 & 45.6 & 43.4 & 40.8 & 44.0 & 46.3 & & 46.2 \\
Random & 25\% & & 46.1 & 42.7 & 38.4 & 39.1 & 35.6 & 39.9 & 40.0 & 40.4 & & 40.4 \\
DyCoke (V\&A) & 25\% & & 46.5 & 45.0 & 42.1 & 38.5 & 40.1 & 39.1 & 38.4 & 42.6 & & 41.8 \\
\base & 25\% & & \textbf{50.0} & 46.6 & 43.2 & 41.7 & 41.6 & \textbf{44.6} & 39.5 & \textbf{45.1} & & 44.1 \\
\base & 20\% & & 48.8 & 44.3 & 40.9 & 40.6 & 38.6 & 40.8 & 38.6 & 43.8 & & 42.3 \\
\rowcolor{OursBlue}
\mbox{\method (Ours)} & 25\% & & 49.8 & \textbf{48.9} & \textbf{44.2} & \textbf{43.8} & 41.6 & 43.3 & \textbf{40.9} & 43.8 & & \textbf{44.7} \\
\rowcolor{OursBlue}
\mbox{\method (Ours)} & 20\% & & 48.0 & 45.6 & 43.2 & 42.0 & \textbf{42.3} & 43.3 & 37.2 & 43.8 & & 43.2 \\
\bottomrule
\end{tabular}
}
\end{table}

\subsubsection{Token Pruning under Allocated Budgets}
\label{sec:budget}

The previous levels determine how many tokens each local unit should retain.
Following \base~\citep{tao2025omnizip}, the pruning stage then selects concrete tokens under these local budgets.
For audio, token importance is estimated from the final self-attention of the audio encoder:
\begin{equation}
    \mathbf{A}
    =
    \operatorname{Softmax}
    \left(
    \frac{\mathbf{Q}_{\mathrm{a}}\mathbf{K}_{\mathrm{a}}^{\top}}{\sqrt{d}}
    \right),
    \qquad
    \alpha_i
    =
    \frac{1}{N_{\mathrm{a}}}
    \sum_{j=1}^{N_{\mathrm{a}}}\mathbf{A}_{j,i}.
    \label{eq:audio_attention_score}
\end{equation}
Within each audio segment, tokens with higher attention scores $\alpha_i$ are retained first.
The retained audio mask is then aggregated over temporally aligned audio chunks, where each chunk corresponds to 2 seconds of audio tokens.
The chunk-level retention ratios provide an audio-guided prior for the corresponding video windows.

For video, pruning is performed within local time windows, each corresponding to four video frames, using interleaved spatio-temporal compression (ISTC).
ISTC alternates temporal pruning, which removes same-position tokens with high similarity across adjacent frames, and spatial pruning, which uses DPC-KNN to preserve representative tokens within each frame.
These pruning rules provide one concrete token-selection backend: they decide which tokens to keep under the allocated budgets, while \method determines the budgets themselves and can be combined with other pruning strategies.

\section{Experiments}
\label{sec:experiments}

\subsection{Experimental Setting}
\label{sec:experimental_setting}

\begin{wrapfigure}{r}{0.5\textwidth}
    \vspace{-1em}
    \centering
    \includegraphics[width=\linewidth]{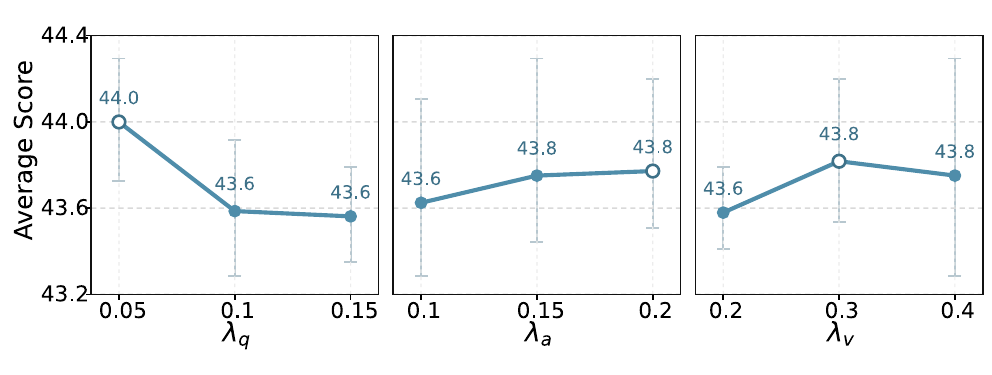}
    \vspace{-2em}
    \caption{Shift-ratio sensitivity on WorldSense with Qwen2.5-Omni-7B. Each point averages over the remaining two coefficients, and the shaded error bar denotes the min--max range.}
    \label{fig:lambda_sensitivity}
\end{wrapfigure}

\paragraph{Benchmarks.}
We evaluate \method on four audio-video QA benchmarks: WorldSense~\citep{hong2026worldsense}, AVUT~\citep{yang2025avut}, VideoMME with audio~\citep{fu2024videomme}, and DailyOmni~\citep{zhou2025dailyomni}.
WorldSense contains 3,172 QA pairs across eight real-world domains, while AVUT is an audio-centric benchmark covering event localization, object/OCR matching, information extraction, content counting, and character matching.
VideoMME evaluates perception and temporal reasoning over short, medium, and long videos, and DailyOmni focuses on temporally aligned audio-video reasoning in daily-life scenarios.

\paragraph{Comparison Methods.}
We compare \method with the full-token Qwen2.5-Omni baseline and three compression baselines: Random pruning, which uniformly drops audio and video tokens; DyCoke~\citep{tao2025dycoke}, whose TTM module is applied to both modalities independently; and \base~\citep{tao2025omnizip}, an audio-guided OmniLLM compressor that uses audio saliency to guide video pruning.

\paragraph{Implementation Details.}
All experiments use Qwen2.5-Omni-7B and Qwen2.5-Omni-3B~\citep{qwen2025qwen25omni} on NVIDIA H20 GPUs, with a maximum of 512 frames for VideoMME and 128 frames for the other benchmarks.
Each audio segment contains 25 audio tokens, and each video frame contains 72 video tokens.
Under the 25\% retained-token setting, the prior budget is initialized from the \base allocation with a 50\% audio keep rate and a 20\% video keep rate.
We set $(\lambda_q,\lambda_a,\lambda_v)$ to $(0.05,0.20,0.30)$ for the 7B model and $(0.15,0.10,0.40)$ for the 3B model.

\begin{table}[!t]
\centering
\caption{Results on AVUT, VideoMME, and DailyOmni under 25\% and 20\% retained-token settings. VideoMME and DailyOmni include their internal splits; AVUT is reported by overall accuracy. Best compressed results, excluding Full Tokens, are bolded.}
\label{tab:multi_benchmark_main}
\scriptsize
\setlength{\tabcolsep}{2.2pt}
\resizebox{\paperfloatwidth}{!}{%
\begin{tabular}{L{2.16cm}C{0.66cm}C{0.08cm}C{0.58cm}C{0.20cm}C{0.62cm}C{0.56cm}C{0.58cm}C{0.16cm}C{0.58cm}C{0.30cm}C{0.58cm}C{0.56cm}C{0.56cm}C{0.58cm}C{0.56cm}C{0.56cm}C{0.16cm}C{0.58cm}}
\toprule
\multirow{2}{*}{\textbf{Method}} & \multirow{2}{*}{\textbf{Ret.}} & & \multirow{2}{*}{\textbf{AVUT}} & & \multicolumn{5}{c}{\textbf{VideoMME}} & & \multicolumn{8}{c}{\textbf{DailyOmni}} \\
\cmidrule(lr){6-10}\cmidrule(lr){12-19}
& & & & & \textbf{Short} & \textbf{Med.} & \textbf{Long} & & \textbf{Avg.} & & \textbf{Con.} & \textbf{Evt.} & \textbf{AV-E} & \textbf{Com.} & \textbf{Inf.} & \textbf{Rea.} & & \textbf{Avg.} \\
\midrule
\multicolumn{19}{c}{\emph{Qwen2.5-Omni-7B}} \\
\midrule
\rowcolor{FullTokenBlue}
Full Tokens & 100\% & & 63.8 & & 77.4 & 68.8 & 55.7 & & 67.3 & & 58.5 & 56.9 & 48.7 & 70.2 & 79.2 & 76.6 & & 62.7 \\
Random & 25\% & & 56.5 & & 72.3 & 64.8 & 54.7 & & 63.9 & & 48.2 & 46.1 & 38.2 & 61.1 & 68.8 & 65.7 & & 52.3 \\
DyCoke (V\&A) & 25\% & & 57.2 & & 71.4 & 63.9 & 54.1 & & 63.1 & & 46.1 & 46.4 & 42.4 & 64.9 & 70.1 & 71.4 & & 54.3 \\
\base & 25\% & & 60.7 & & \textbf{74.9} & 68.1 & \textbf{56.0} & & 66.3 & & 48.7 & 47.7 & 44.1 & \textbf{68.7} & 79.2 & 72.0 & & 57.1 \\
\base & 20\% & & 59.9 & & 74.2 & 66.6 & 55.0 & & 65.3 & & 45.6 & 48.4 & 43.3 & 66.4 & 77.3 & 71.4 & & 56.0 \\
\rowcolor{OursBlue}
\mbox{\method (Ours)} & 25\% & & \textbf{61.1} & & 74.7 & \textbf{68.6} & \textbf{56.0} & & \textbf{66.4} & & \textbf{50.8} & 49.3 & \textbf{44.5} & \textbf{68.7} & \textbf{79.9} & \textbf{75.4} & & \textbf{58.5} \\
\rowcolor{OursBlue}
\mbox{\method (Ours)} & 20\% & & 59.9 & & 73.7 & 67.7 & 55.8 & & 65.7 & & 49.2 & \textbf{50.3} & 42.0 & 64.9 & 77.9 & 73.1 & & 57.0 \\
\midrule
\multicolumn{19}{c}{\emph{Qwen2.5-Omni-3B}} \\
\midrule
\rowcolor{FullTokenBlue}
Full Tokens & 100\% & & 61.6 & & 74.6 & 64.9 & 52.3 & & 63.9 & & 53.9 & 52.9 & 51.7 & 67.2 & 76.6 & 71.4 & & 60.2 \\
Random & 25\% & & 54.8 & & 68.7 & 60.0 & 49.4 & & 59.4 & & 46.6 & 44.8 & 42.0 & 60.3 & 66.2 & 59.4 & & 51.1 \\
DyCoke (V\&A) & 25\% & & 55.1 & & 67.8 & 59.7 & 49.8 & & 59.1 & & 45.6 & 43.8 & 41.6 & 61.1 & 64.9 & 62.9 & & 51.0 \\
\base & 25\% & & \textbf{58.9} & & 72.0 & 61.6 & \textbf{52.0} & & 61.9 & & 48.2 & \textbf{46.4} & 46.6 & 64.9 & 72.1 & 68.0 & & 55.2 \\
\base & 20\% & & 56.1 & & 71.1 & 60.9 & 51.7 & & 61.2 & & 48.2 & 43.8 & 43.7 & 65.6 & 73.4 & 68.0 & & 54.2 \\
\rowcolor{OursBlue}
\mbox{\method (Ours)} & 25\% & & 58.8 & & \textbf{72.2} & \textbf{63.6} & 51.4 & & \textbf{62.4} & & 49.2 & 44.1 & \textbf{47.9} & \textbf{67.2} & \textbf{74.7} & \textbf{70.9} & & \textbf{56.1} \\
\rowcolor{OursBlue}
\mbox{\method (Ours)} & 20\% & & 57.0 & & 70.8 & 62.6 & 51.0 & & 61.4 & & \textbf{49.7} & 45.4 & 45.8 & \textbf{67.2} & 72.1 & 66.9 & & 55.1 \\
\bottomrule
\end{tabular}
}
\end{table}

\subsection{Main Results}
\label{sec:main_results}

\paragraph{Comparison with Baselines.}
Tables~\ref{tab:worldsense_main} and~\ref{tab:multi_benchmark_main} compare \method with full-token inference and compressed baselines.
Random pruning and DyCoke drop sharply under both 25\% and 20\% retained-token budgets, while \base is stronger but still suffers from aggressive compression, indicating that a fixed audio/video prior and window-local frame budgets can misplace scarce tokens when the required evidence varies across queries and samples.
\method improves the WorldSense average over \base on both model sizes and retained ratios, from 43.2 to 44.2 at 25\% and from 42.7 to 43.0 at 20\% for the 7B model, and from 44.1 to 44.7 at 25\% and from 42.3 to 43.2 at 20\% for the 3B model.
On AVUT, VideoMME, and DailyOmni, it also gives the best or competitive compressed aggregate results across most settings, including the strongest 7B averages on all three benchmarks at 25\% and the strongest 3B VideoMME and DailyOmni averages at 25\%.
These results show that query-guided, fine-grained budget redistribution is complementary to token selection: before deciding which tokens to keep, \method improves where the limited audio-visual budget is spent.

\begin{wrapfigure}{r}{0.5\textwidth}
    \centering
    \includegraphics[width=\linewidth]{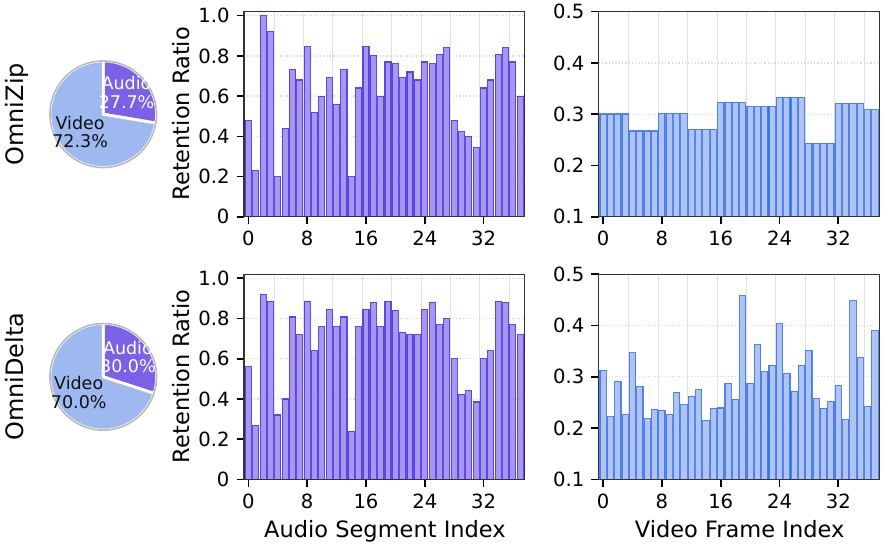}
    \caption{Visualization of budget allocation. \method assigns non-uniform budgets across audio segments and video frames, retaining more tokens in informative temporal regions and compressing redundant regions more aggressively.}
    \label{fig:budget_visualization}
    \vspace{0em}
\end{wrapfigure}

\paragraph{Shift-Ratio Sensitivity.}
Figure~\ref{fig:lambda_sensitivity} studies the three shift coefficients on WorldSense with Qwen2.5-Omni-7B.
The sweep uses $\lambda_q\in\{0.05,0.10,0.15\}$, $\lambda_a\in\{0.10,0.15,0.20\}$, and $\lambda_v\in\{0.20,0.30,0.40\}$, producing 27 configurations in total.
Across this grid, the average accuracy varies within a narrow range of 43.28\%--44.29\%.
Among the three coefficients, the query-level shift is the most sensitive component, indicating that inter-modal transfer strength should be calibrated more carefully than the intra-modal shifts.
The audio and video internal shifts are more stable, with moderate video shifts peaking around $\lambda_v=0.30$ on average and audio shifts showing a flatter response.
Overall, the sweep supports a bounded-shift design: performance improves when the transfer range is calibrated, and all tested settings remain above the 43.2\% \base average in Table~\ref{tab:worldsense_main}.

\paragraph{Budget Allocation Visualization.}
Figure~\ref{fig:budget_visualization} provides a case-level view of the budget distribution produced by \method.
The audio branch allocates different keep budgets to different temporal segments, while the video branch further redistributes the visual budget at the frame level rather than keeping a fixed budget inside each time window.
This adaptive allocation concentrates tokens around segments and frames that are more likely to contain task-relevant evidence, and assigns fewer tokens to redundant or less informative intervals.
The visualization illustrates that \method improves compression not by changing the downstream pruning rule, but by placing the limited audio-video budget more precisely before token selection.

\subsection{Ablation Study}
\label{sec:ablation}

\begin{table*}[!t]
\centering
\begin{minipage}[t]{0.525\textwidth}
\centering
\caption{Ablation of the three budget-allocation levels on Qwen2.5-Omni-7B under the 25\% retained-token setting. A/V, A, and V denote audio-video, audio-internal, and video-internal budget allocation. Dataset columns correspond to WorldSense, AVUT, VideoMME, and DailyOmni; values in parentheses are deltas from the baseline.}
\label{tab:ablation_budget}
\scriptsize
\setlength{\tabcolsep}{0.8pt}
\begin{tabularx}{\linewidth}{@{}L{1.38cm} *{3}{C{0.34cm}} *{5}{>{\centering\arraybackslash}X}@{}}
\toprule
\multirow{2}{*}{\textbf{ID}} & \multicolumn{3}{c}{\textbf{Setting}} & \multicolumn{4}{c}{\textbf{Dataset}} & \multirow{2}{*}{\textbf{Avg.}} \\
\cmidrule(lr){2-4}\cmidrule(lr){5-8}
& \textbf{A/V} & \textbf{A} & \textbf{V} & \textbf{WS} & \textbf{AVUT} & \textbf{VMME} & \textbf{Daily} & \\
\midrule
\rowcolor{FullTokenBlue}
Baseline & \textbf{\ding{55}} & \textbf{\ding{55}} & \textbf{\ding{55}} & 43.2 & 60.7 & 66.3 & 57.1 & 56.8 \\
\multirow{4}{*}{\mbox{\method}} & \textbf{\ding{51}} & \textbf{\ding{55}} & \textbf{\ding{55}} & 44.0{\tiny(+0.8)} & 60.3{\tiny(-0.4)} & 65.8{\tiny(-0.5)} & 57.4{\tiny(+0.3)} & 56.9{\tiny(+0.1)} \\
& \textbf{\ding{51}} & \textbf{\ding{51}} & \textbf{\ding{55}} & 43.8{\tiny(+0.6)} & 60.6{\tiny(-0.1)} & 66.3{\tiny(+0.0)} & 57.6{\tiny(+0.5)} & 57.1{\tiny(+0.3)} \\
& \textbf{\ding{51}} & \textbf{\ding{55}} & \textbf{\ding{51}} & 44.1{\tiny(+0.9)} & 60.7{\tiny(+0.0)} & 66.6{\tiny(+0.3)} & 57.8{\tiny(+0.7)} & 57.3{\tiny(+0.5)} \\
\rowcolor{OursBlue}
& \textbf{\ding{51}} & \textbf{\ding{51}} & \textbf{\ding{51}} & 44.2{\tiny(+1.0)} & 61.1{\tiny(+0.4)} & 66.4{\tiny(+0.1)} & 58.5{\tiny(+1.4)} & 57.6{\tiny(+0.7)} \\
\bottomrule
\end{tabularx}
\end{minipage}
\hfill
\begin{minipage}[t]{0.465\textwidth}
\centering
\caption{Actual inference efficiency comparison on WorldSense. Time to first token (TTFT) and end-to-end latency are measured per example.}
\label{tab:efficiency_analysis}
\scriptsize
\setlength{\tabcolsep}{0.2pt}
\begin{tabularx}{\linewidth}{@{}l X c c X c c@{}}
\toprule
\textbf{Method} & & \textbf{GPU Mem.(G)} $\downarrow$ & \textbf{TTFT (ms)} $\downarrow$ & & \textbf{Acc.} $\uparrow$ & \textbf{E2E Lat. (s)} $\downarrow$ \\
\midrule
\multicolumn{7}{c}{\emph{Qwen2.5-Omni-7B}} \\
\midrule
\rowcolor{FullTokenBlue}
Full Tokens & & 30.0 & 3006\,(1.00$\times$) & & 46.8 & 3.15\,(1.00$\times$) \\
DyCoke 50\% & & 25.8 & 2091\,(1.44$\times$) & & 44.0 & 2.24\,(1.41$\times$) \\
\base 30\% & & 24.0 & 1809\,(1.66$\times$) & & 44.1 & 1.96\,(1.61$\times$) \\
\rowcolor{OursBlue}
\mbox{\method{} 25\%} & & \textbf{23.4} & \textbf{1776\,(1.69$\times$)} & & \textbf{44.2} & \textbf{1.92\,(1.64$\times$)} \\
\midrule
\multicolumn{7}{c}{\emph{Qwen2.5-Omni-3B}} \\
\midrule
\rowcolor{FullTokenBlue}
Full Tokens & & 18.6 & 2190\,(1.00$\times$) & & 46.2 & 2.25\,(1.00$\times$) \\
DyCoke 50\% & & 15.1 & 1680\,(1.30$\times$) & & 44.5 & 1.74\,(1.29$\times$) \\
\base 30\% & & 13.8 & 1567\,(1.40$\times$) & & 44.5 & 1.63\,(1.38$\times$) \\
\rowcolor{OursBlue}
\mbox{\method{} 25\%} & & \textbf{13.3} & \textbf{1538\,(1.42$\times$)} & & \textbf{44.7} & \textbf{1.60\,(1.41$\times$)} \\
\bottomrule
\end{tabularx}
\end{minipage}
\end{table*}

Table~\ref{tab:ablation_budget} studies the contribution of each budget-allocation level.
The top-level audio-video allocation, audio-internal allocation, and video-internal allocation all improve the model from different perspectives: they adjust the modality split, redistribute audio evidence across temporal segments, and refine visual budgets across frames.
Combining the three modules gives the best overall average, indicating that the hierarchical design is effective and that the three allocation levels are complementary.

\subsection{Efficiency Analysis}
\label{sec:efficiency_analysis}

We further evaluate actual GPU memory consumption, time to first token (TTFT), and end-to-end latency on WorldSense.
As shown in Table~\ref{tab:efficiency_analysis}, \method substantially reduces inference cost compared with full-token inference while preserving strong accuracy.
On Qwen2.5-Omni-7B, it reduces GPU memory from 30.0G to 23.4G and achieves a 1.69$\times$ TTFT speedup together with a 1.64$\times$ end-to-end latency speedup.
On Qwen2.5-Omni-3B, it also obtains the lowest memory usage and latency among compressed methods, with a 1.42$\times$ TTFT speedup and a 1.41$\times$ latency speedup.
Among the compressed baselines, \method achieves the highest accuracy while also providing the largest acceleration, indicating that the proposed budget allocation offers a better trade-off between practical efficiency and model performance.

\section{Related Work}

\subsection{Omni-Modal Large Language Models}

Omni-modal large language models (OmniLLMs) extend multimodal LLMs from image/video-text inputs to unified processing of text, vision, speech, and audio.
Recent video and audio-video LLMs have strengthened temporal reasoning and audio-aware video understanding \citep{cheng2024videollama2,sun2024videosalmonn,zhang2024llavavideo,tang2025videosalmonn2}.
OmniLLMs further integrate heterogeneous inputs into one autoregressive interface, with recent Omni-series and open-source omni-modal models demonstrating strong audio-video interaction capabilities \citep{fu2024vita,li2024baichuanomni,liu2025ola,inclusion2025mingomni,tong2025interactiveomni,ye2025omnivinci,qwen2025qwen25omni,qwen2025qwen3omni}.
This unified sequence lets the model combine complementary visual and acoustic evidence, but long videos and dense audio streams also dominate prefill computation and KV-cache memory.
Efficient OmniLLM inference therefore requires compressing audio-video tokens while preserving cross-modal correspondence.

\subsection{Token Pruning for OmniLLMs}

Token pruning lowers multimodal inference cost by removing redundant perceptual tokens before the LLM.
Image methods use token merging, early-layer filtering, and sparsification~\citep{bolya2022tome,chen2024fastv,shang2024llavaprumerge,zhang2024sparsevlm,ye2024fitprune}, video methods exploit temporal saliency, frame similarity, and hierarchical spatio-temporal structure~\citep{huang2025prunevid,tao2025dycoke,fu2025framefusion,shen2024longvu,guo2026hieravid}, and audio methods use adjacent merging, importance pruning, and context-aware pruning~\citep{li2023adjacentmerging,lee2025audiopruning,lin2025speechprune}.
They reveal substantial perceptual redundancy but usually operate within a single stream under a fixed budget.

OmniLLM compressors extend pruning across audio and video: training-free \base uses audio attention to guide video pruning~\citep{tao2025omnizip}; OmniSIFT and EchoingPixels learn cross-modal selectors with straight-through estimators and end-to-end training~\citep{ding2026omnisift,gong2025echoingpixels}; FastAV, AccKV, and DASH compress audio-video tokens, KV caches, or semantic chunks within predefined pipelines~\citep{jung2026fastav,jiang2026acckv,li2026dash}; and OmniRefine performs cooperative pruning over aligned units~\citep{deng2026omnirefine}.
OmniFit follows a fixed layer-wise budget schedule independent of query demand~\citep{wang2026omnifit}, while OmniSelect uses a costly auxiliary AudioCLIP module to estimate modality relevance from direct cosine similarity between the query and audio/video tokens and adjust pruning~\citep{yang2026omniselect,guzhov2021audioclip}.
These methods primarily decide which tokens, chunks, or caches to retain.
\method instead addresses the preceding allocation problem by distributing a fixed budget across modalities, audio segments, and video frames, conditioning the inter-modal split on the query while remaining compatible with existing token-selection rules.

\section{Conclusion}

We proposed \method, a training-free hierarchical budget allocation framework for OmniLLM token compression.
Instead of only selecting important tokens, \method decides where a fixed retained-token budget should be spent, using query-skill similarity for audio-video budget shifting and local redundancy/complexity for audio segment and video frame allocation, while remaining compatible with existing pruning strategies.
Experiments on four audio-video benchmarks with Qwen2.5-Omni-7B and Qwen2.5-Omni-3B show that \method achieves the best compressed accuracy, reduces GPU memory, and provides the largest inference acceleration, including a 1.64$\times$ end-to-end speedup.
These results indicate that budget allocation before token selection is an effective direction for efficient omni-modal inference.

\bibliographystyle{iclr2026_conference}
\bibliography{iclr2026_conference}

\clearpage
\appendix

\begin{center}
    {\Large\bfseries Appendix}
\end{center}

\setcounter{equation}{0}
\renewcommand{\theequation}{\Roman{equation}}
\setcounter{table}{0}
\renewcommand{\thetable}{\Roman{table}}

This appendix supplements the main paper as follows:
\begin{itemize}[leftmargin=*]
    \item Sec.~\ref{app:algorithm} provides the pseudocode for dynamic budget allocation in \method.
    \item Sec.~\ref{app:theory} establishes the boundedness, feasibility, and conservation properties of the allocator.
    \item Sec.~\ref{app:complexity} analyzes its computational and memory complexity.
    \item Sec.~\ref{app:motivation_details} provides prompts, examples, and additional results for the two diagnostic studies in Sec.~\ref{sec:motivation}.
    \item Sec.~\ref{app:skill_pool} describes the construction of the modality skill pools and presents representative entries.
\end{itemize}

\section{Detailed Algorithm}
\label{app:algorithm}

Algorithm~\ref{alg:omnidelata} summarizes how \method distributes a fixed retained-token budget between modalities and then among their temporal units.

\begin{algorithm}[H]
\caption{Dynamic Budget Allocation}
\label{alg:omnidelata}
\KwIn{Query $q$; audio/video skill pools $\mathcal{S}_{\mathrm{a}},\mathcal{S}_{\mathrm{v}}$; projected tokens $\mathbf{A},\mathbf{V}$; target retained ratio $r$; prior modality budgets $(K_{\mathrm{a}}^0,K_{\mathrm{v}}^0)$; shift coefficients $\lambda_q,\lambda_{\mathrm{a}},\lambda_{\mathrm{v}}$.}
\KwOut{Modality budgets $(K_{\mathrm{a}},K_{\mathrm{v}})$ and local keep counts $\{k_i^m\}$ satisfying $\sum_{m,i}k_i^m=K$.}
$K\leftarrow\operatorname{round}\!\left(r(N_{\mathrm{a}}+N_{\mathrm{v}})\right)$\;
\tcp{Stage 1: Query-aware inter-modal budget allocation}
Encode $q$ and all skills with mean-pooled, $\ell_2$-normalized thinker input embeddings $\mathbf{e}(\cdot)$\;
\ForEach{$m\in\{\mathrm{a},\mathrm{v}\}$}{
    $r_m\leftarrow k^{-1}\sum_{c\in\operatorname{TopK}_k(q,\mathcal{S}_m)}\mathbf{e}(q)^\top\mathbf{e}(c)$\;
}
$[p_{\mathrm{a}},p_{\mathrm{v}}]\leftarrow\operatorname{Softmax}([r_{\mathrm{a}},r_{\mathrm{v}}])$, $m_q\leftarrow p_{\mathrm{v}}-p_{\mathrm{a}}$, $\Delta_q\leftarrow\lambda_qKm_q$\;
$(K_{\mathrm{a}},K_{\mathrm{v}})\leftarrow\operatorname{ProjectExact}(K_{\mathrm{a}}^0-\Delta_q,K_{\mathrm{v}}^0+\Delta_q;K)$\;
\tcp{Stage 2: Fine-grained intra-modal budget allocation}
Partition $\mathbf{A}$ into one-second segments and $\mathbf{V}$ into individual frames, producing units $\{u_i^m\}_{i=1}^{n_m}$\;
\ForEach{$m\in\{\mathrm{a},\mathrm{v}\}$}{
    $\{k_i^{0,m}\}\leftarrow\operatorname{PriorLocalBudget}(m,K_m)$ such that $\sum_i k_i^{0,m}=K_m$\;
    \ForEach{$u_i^m=\{\mathbf{x}_{i,j}\}_{j=1}^{L_i}$}{
        $\boldsymbol{\mu}_i\leftarrow L_i^{-1}\sum_j\mathbf{x}_{i,j}$, $C_i^m\leftarrow\operatorname{Norm}\!\left(1-L_i^{-1}\sum_j\cos(\mathbf{x}_{i,j},\boldsymbol{\mu}_i)\right)$\;
        $R_i^m\leftarrow 0$ if $i=1$; otherwise $R_i^m\leftarrow\operatorname{Norm}(\operatorname{AlignedCos}(u_{i-1}^m,u_i^m))$\;
    }
    $z_i^m\leftarrow\operatorname{Clip}\!\left((R_i^m-C_i^m)-n_m^{-1}\sum_j(R_j^m-C_j^m),-1,1\right)$ for all $i$\;
    \ForEach{$i\in\{1,\ldots,n_m\}$}{
        $\ell_i^m\leftarrow k_i^{0,m}-\lambda_m[z_i^m]_+L_i$, $h_i^m\leftarrow k_i^{0,m}+\lambda_m[-z_i^m]_+L_i$, $w_i^m\leftarrow(1-z_i^m)/2$\;
    }
    $\hat{k}_i^m\leftarrow\ell_i^m+\left(K_m-\sum_j\ell_j^m\right)\frac{w_i^m(h_i^m-\ell_i^m)}{\sum_jw_j^m(h_j^m-\ell_j^m)}$\;
    $\{k_i^m\}\leftarrow\operatorname{LargestRemainder}(\{\hat{k}_i^m\},\{\ell_i^m,h_i^m\},K_m)$\;
}
\KwRet{$K_{\mathrm{a}},K_{\mathrm{v}},\{k_i^{\mathrm{a}}\},\{k_i^{\mathrm{v}}\}$}\;
\end{algorithm}

Here, $\operatorname{ProjectExact}$ projects the modality budgets into their feasible ranges and performs integer correction while preserving $K_{\mathrm{a}}+K_{\mathrm{v}}=K$.
$\operatorname{PriorLocalBudget}$ initializes audio-segment budgets from audio attention and video-frame budgets from the audio-guided temporal prior.
$\operatorname{AlignedCos}$ averages the cosine similarities between corresponding token representations in adjacent units.
$\operatorname{LargestRemainder}$ converts the continuous allocations into bounded integer counts whose sum is exactly $K_m$.

\section{Theoretical Properties of Dynamic Budget Allocation}
\label{app:theory}

We analyze the two-stage allocator independently of its downstream pruning operator.
The results below establish that query-aware transfer is smooth and bounded, local redistribution follows the intended direction, and integer correction preserves the exact retained-token budget.

\subsection{Bounded and Monotonic Inter-Modal Transfer}

Let $g=r_{\mathrm{v}}-r_{\mathrm{a}}$ denote the difference between the video and audio skill scores.
Because the router applies a two-class Softmax, its signed modality bias has the closed form
\begin{equation}
    m_q
    =p_{\mathrm{v}}-p_{\mathrm{a}}
    =\frac{e^{r_{\mathrm{v}}}-e^{r_{\mathrm{a}}}}
    {e^{r_{\mathrm{v}}}+e^{r_{\mathrm{a}}}}
    =\tanh\left(\frac{g}{2}\right).
    \label{eq:appendix_softmax_margin}
\end{equation}

\begin{proposition}[Bounded and monotonic query transfer]
For $K>0$ and $\lambda_q>0$, before range projection and integer correction, the query-driven transfer
$\Delta_q=\lambda_qK\tanh(g/2)$ is an odd and strictly increasing function of $g$ satisfying
\begin{equation}
    |\Delta_q|\leq\lambda_qK,
    \qquad
    \widetilde K_{\mathrm{a}}+\widetilde K_{\mathrm{v}}=K,
    \label{eq:appendix_query_bounds}
\end{equation}
where $\widetilde K_{\mathrm{a}}=K_{\mathrm{a}}^0-\Delta_q$ and
$\widetilde K_{\mathrm{v}}=K_{\mathrm{v}}^0+\Delta_q$.
\end{proposition}

\begin{proof}
The hyperbolic tangent is odd, strictly increasing, and bounded in $(-1,1)$, which gives the first property and the bound on $\Delta_q$.
Since the same quantity is subtracted from the audio prior and added to the video prior,
$\widetilde K_{\mathrm{a}}+\widetilde K_{\mathrm{v}}=K_{\mathrm{a}}^0+K_{\mathrm{v}}^0=K$.
The subsequent $\operatorname{ProjectExact}$ operation only enforces feasible modality ranges and integer counts while preserving this sum.
\end{proof}

The transfer is also stable with respect to changes in the score gap.
Since
\begin{equation}
    \left|\frac{\partial\Delta_q}{\partial g}\right|
    =\frac{\lambda_qK}{2}\operatorname{sech}^2\left(\frac{g}{2}\right)
    \leq\frac{\lambda_qK}{2},
    \label{eq:appendix_query_lipschitz}
\end{equation}
the mean value theorem gives the following result.

\begin{corollary}[Stability of query routing]
For any two score gaps $g_1$ and $g_2$,
\begin{equation}
    |\Delta_q(g_1)-\Delta_q(g_2)|
    \leq\frac{\lambda_qK}{2}|g_1-g_2|.
    \label{eq:appendix_query_stability}
\end{equation}
Moreover, when $g$ is close to zero,
$\Delta_q=\lambda_qK(g/2)+\mathcal{O}(g^3)$.
Thus, an ambiguous query induces only a small transfer, whereas a clear modality preference produces a larger but still bounded adjustment.
\end{corollary}

\subsection{Directional and Bounded Intra-Modal Redistribution}

For this analysis, let
\begin{equation}
    \widetilde z_i^m
    =(R_i^m-C_i^m)
    -\frac{1}{n_m}\sum_{j=1}^{n_m}(R_j^m-C_j^m),
    \qquad
    z_i^m=\operatorname{Clip}(\widetilde z_i^m,-1,1),
    \label{eq:appendix_centered_score}
\end{equation}
where the clipped score is the bounded implementation used in Algorithm~\ref{alg:omnidelata}.
Centering removes the common bias shared by all units, while clipping keeps the priority
$w_i^m=(1-z_i^m)/2$ in $[0,1]$.

\begin{proposition}[Direction and magnitude of local shifts]
Assume the original interval $[\ell_i^m,h_i^m]$ is feasible and let
$k_i^m\in[\ell_i^m,h_i^m]$.
Then
\begin{equation}
    z_i^m>0 \Rightarrow k_i^m\leq k_i^{0,m},
    \qquad
    z_i^m<0 \Rightarrow k_i^m\geq k_i^{0,m},
    \label{eq:appendix_shift_direction}
\end{equation}
and in both cases
\begin{equation}
    |k_i^m-k_i^{0,m}|\leq\lambda_mL_i.
    \label{eq:appendix_local_bound}
\end{equation}
\end{proposition}

\begin{proof}
If $z_i^m>0$, then $[-z_i^m]_+=0$ and hence
$h_i^m=k_i^{0,m}$, while
$\ell_i^m=k_i^{0,m}-\lambda_m z_i^mL_i$.
The unit can therefore only lose budget, by at most $\lambda_mL_i$.
If $z_i^m<0$, then $[z_i^m]_+=0$, so
$\ell_i^m=k_i^{0,m}$ and
$h_i^m=k_i^{0,m}+\lambda_m(-z_i^m)L_i$.
The unit can only gain budget, again by at most $\lambda_mL_i$.
For $z_i^m=0$, the interval collapses to the prior count.
\end{proof}

This property gives the complexity-redundancy score a direct allocation interpretation.
A unit with high temporal redundancy relative to its internal complexity has $z_i^m>0$ and can be compressed more aggressively, whereas a complex and less repetitive unit has $z_i^m<0$ and receives additional capacity.
If the original intervals cannot contain the exact modality budget, the implementation minimally relaxes their boundaries before integer allocation; this repair preserves token-count feasibility and is needed only at boundary cases.

\subsection{Exact Integer Budget Conservation}

After interval construction and any required feasibility repair, let
$\ell_i^m,h_i^m\in\mathbb{Z}$ satisfy
\begin{equation}
    \sum_i\ell_i^m\leq K_m\leq\sum_i h_i^m.
    \label{eq:appendix_interval_feasibility}
\end{equation}
Define the residual budget and capacity by
\begin{equation}
    B_m=K_m-\sum_i\ell_i^m,
    \qquad
    c_i^m=h_i^m-\ell_i^m.
    \label{eq:appendix_residual_capacity}
\end{equation}

\begin{proposition}[Exact bounded integer allocation]
The bounded largest-remainder procedure returns integer counts $k_i^m$ satisfying
\begin{equation}
    \ell_i^m\leq k_i^m\leq h_i^m,
    \qquad
    \sum_i k_i^m=K_m.
    \label{eq:appendix_exact_local_budget}
\end{equation}
\end{proposition}

\begin{proof}
The procedure starts from the lower bounds and writes
$k_i^m=\ell_i^m+x_i^m$, where $x_i^m$ is an integer initialized to zero.
Equation~\eqref{eq:appendix_interval_feasibility} implies
$0\leq B_m\leq\sum_i c_i^m$.
Each allocation step increases only an index with $x_i^m<c_i^m$ and therefore preserves
$0\leq x_i^m\leq c_i^m$.
If fewer than $B_m$ residual tokens have been assigned, the unused aggregate capacity is positive, so at least one feasible index remains.
Consequently, the procedure terminates with $\sum_i x_i^m=B_m$.
Substitution into $k_i^m=\ell_i^m+x_i^m$ proves both the interval constraints and the exact sum.
\end{proof}

\begin{corollary}[Sample-level retained-ratio guarantee]
Combining exact local allocation with the inter-modal projection gives
\begin{equation}
    \sum_i k_i^{\mathrm{a}}+\sum_i k_i^{\mathrm{v}}
    =K_{\mathrm{a}}+K_{\mathrm{v}}
    =K
    =\operatorname{round}\left(r(N_{\mathrm{a}}+N_{\mathrm{v}})\right).
    \label{eq:appendix_exact_global_budget}
\end{equation}
Hence, different queries and local redundancy patterns change only the placement of the budget, not the sample-level retained ratio.
\end{corollary}

\subsection{Why Heterogeneous Budgets Are Necessary}

We finally give an optimization interpretation of fine-grained allocation.
Let $D_i^m(k)$ denote the information loss of unit $u_i^m$ after retaining $k$ tokens.

\begin{assumption}[Diminishing returns]
For each local unit, $D_i^m(k)$ is differentiable, decreasing, and convex in $k$.
Thus, retaining more tokens cannot increase information loss, while the benefit of each additional token gradually diminishes.
\end{assumption}

Under this assumption, the ideal continuous allocation solves
\begin{equation}
    \min_{\{k_i^m\}}
    \sum_{i=1}^{n_m}D_i^m(k_i^m)
    \quad
    \text{s.t.}
    \quad
    \sum_i k_i^m=K_m,
    \quad
    \ell_i^m\leq k_i^m\leq h_i^m.
    \label{eq:appendix_loss_optimization}
\end{equation}

\begin{proposition}[Uniform allocation under heterogeneous marginal gains]
Suppose the uniform allocation $\bar{k}=K_m/n_m$ lies in the interior of all feasible intervals.
If there exist units $i$ and $j$ such that
\begin{equation}
    -(D_i^m)'(\bar{k})\neq-(D_j^m)'(\bar{k}),
    \label{eq:appendix_unequal_marginal_gain}
\end{equation}
then the uniform allocation is not optimal for Eq.~\eqref{eq:appendix_loss_optimization}.
\end{proposition}

\begin{proof}
Without loss of generality, assume
$-(D_i^m)'(\bar{k})>-(D_j^m)'(\bar{k})$.
Transfer a sufficiently small budget $\varepsilon>0$ from unit $j$ to unit $i$.
The first-order change in total loss is
\begin{equation}
    \Delta D
    =\varepsilon(D_i^m)'(\bar{k})
    -\varepsilon(D_j^m)'(\bar{k})
    +o(\varepsilon)<0.
    \label{eq:appendix_loss_decrease}
\end{equation}
Thus, the transfer strictly decreases total information loss while preserving the total budget, contradicting the optimality of the uniform allocation.
At an interior optimum, the Karush--Kuhn--Tucker conditions instead require all active units to have equal marginal gains.
\end{proof}

The loss functions are not directly observable during inference.
\method uses local complexity and temporal redundancy as training-free proxies for their marginal gains: high complexity and low redundancy indicate that removing another token is more likely to lose information, while repetitive units can absorb more compression.
This proposition motivates non-uniform allocation but does not assume that the proposed signals recover the unknown optimal loss functions exactly.

\section{Computational and Memory Complexity}
\label{app:complexity}

\subsection{Notation}

Let $N=N_{\mathrm{a}}+N_{\mathrm{v}}$ be the number of audio-video tokens before pruning, $N_{\mathrm{t}}$ the number of text and special tokens retained unchanged, $d$ the LLM hidden dimension shared by the projected audio, video, and text token representations, and $L$ the number of LLM layers.
The full and compressed LLM input lengths are denoted by $\bar N=N_{\mathrm{t}}+N$ and $\bar N_r=N_{\mathrm{t}}+rN$, respectively.
Let $S=|\mathcal{S}_{\mathrm{a}}|+|\mathcal{S}_{\mathrm{v}}|$ be the total number of skills, $L_q$ the query length, and $k$ the routing Top-$k$.
We use $U=n_{\mathrm{a}}+n_{\mathrm{v}}$ for the total number of audio segments and video frames, and $L_{\max}$ for the maximum number of tokens in a local unit.

\subsection{Skill-Pool Routing}

Skill embeddings are sample-independent and can be cached before inference.
If a skill contains at most $\bar L_s$ text tokens, constructing the complete bank once costs
\begin{equation}
    T_{\mathrm{skill}}^{\mathrm{cache}}=\mathcal{O}(S\bar L_s d),
    \qquad
    M_{\mathrm{skill}}=\mathcal{O}(Sd).
    \label{eq:appendix_skill_cache_complexity}
\end{equation}
At inference time, mean-pooling the query embeddings costs $\mathcal{O}(L_qd)$, computing its similarities with all cached skills costs $\mathcal{O}(Sd)$, and selecting the Top-$k$ entries costs $\mathcal{O}(S\log k)$.
The online routing overhead is therefore
\begin{equation}
    T_{\mathrm{route}}
    =\mathcal{O}\left((L_q+S)d+S\log k\right),
    \label{eq:appendix_router_complexity}
\end{equation}
which is linear in the skill-bank size and is incurred only once per query.

\subsection{Intra-Modal Signal and Allocation Overhead}

Computing unit means and intra-unit complexity visits each modality token once, giving $\mathcal{O}(Nd)$ time.
Previous-unit redundancy compares aligned token positions between adjacent units; each token participates in only a constant number of such comparisons, so this step is also $\mathcal{O}(Nd)$.
Min-max normalization, score centering, and interval construction require $\mathcal{O}(U)$ operations.

The bounded largest-remainder implementation repeatedly assigns residual counts to units with remaining capacity.
Using a priority operation over $U$ units, its cost is bounded by
$\mathcal{O}(L_{\max}U\log U)$, because no unit can receive more than $L_{\max}$ integer increments.
In Qwen2.5-Omni, $L_{\max}$ is fixed by tokenization: an audio segment contains 25 tokens and a video frame contains approximately 72 tokens.
This term is therefore effectively $\mathcal{O}(U\log U)$ for the models considered in this work.
Combining inter- and intra-modal allocation gives
\begin{equation}
    T_{\mathrm{alloc}}
    =\mathcal{O}\left((N+L_q+S)d+S\log k+L_{\max}U\log U\right).
    \label{eq:appendix_allocator_complexity}
\end{equation}

The projected audio-video features are already produced by the encoders and are shared with the base model.
Excluding these shared features, the allocator stores the cached skill bank and a constant number of statistics per local unit, resulting in
\begin{equation}
    M_{\mathrm{alloc}}=\mathcal{O}(Sd+U).
    \label{eq:appendix_allocator_memory}
\end{equation}

\begin{table}[H]
\centering
\caption{Complexity of the main components in dynamic budget allocation. Skill-bank construction is a one-time cost; the remaining components are evaluated once per sample.}
\label{tab:appendix_complexity}
\small
\setlength{\tabcolsep}{8pt}
\begin{tabular}{lcc}
\toprule
\textbf{Component} & \textbf{Time} & \textbf{Additional Memory} \\
\midrule
Skill-bank construction & $\mathcal{O}(S\bar L_s d)$ & $\mathcal{O}(Sd)$ \\
Online query routing & $\mathcal{O}((L_q+S)d+S\log k)$ & $\mathcal{O}(d)$ \\
Local signal estimation & $\mathcal{O}(Nd)$ & $\mathcal{O}(U)$ \\
Bounded integer allocation & $\mathcal{O}(L_{\max}U\log U)$ & $\mathcal{O}(U)$ \\
\bottomrule
\end{tabular}
\end{table}

\subsection{Overall Inference and Memory Complexity}

For the full sequence of $\bar N$ tokens, the dominant Transformer operations have time complexity
\begin{equation}
    T_{\mathrm{full}}
    =\mathcal{O}\left(L(\bar N^2d+\bar Nd^2)\right),
    \label{eq:appendix_full_complexity}
\end{equation}
where $\bar N^2d$ is the self-attention term and $\bar Nd^2$ covers linear projections and feed-forward layers.
After retaining $rN$ audio-video tokens, the corresponding LLM cost becomes
\begin{equation}
    T_{\mathrm{LLM}}^{\mathrm{compressed}}
    =\mathcal{O}\left(L(\bar N_r^2d+\bar N_r d^2)\right).
    \label{eq:appendix_compressed_complexity}
\end{equation}
Therefore, the complete inference cost can be written as
\begin{equation}
    T_{\mathrm{OmniDelta}}
    =T_{\mathrm{encoder}}+T_{\mathrm{prune}}+T_{\mathrm{alloc}}
    +\mathcal{O}\left(L(\bar N_r^2d+\bar N_r d^2)\right),
    \label{eq:appendix_total_complexity}
\end{equation}
where $T_{\mathrm{prune}}$ depends on the compatible pruning backend and is not introduced by the budget allocator.
The quadratic attention term is reduced by a factor of $(\bar N_r/\bar N)^2$.
When the audio-video sequence dominates the much shorter text sequence, this factor approaches $r^2$; retained ratios of 25\% and 20\% then correspond to approximately 6.25\% and 4\% of the full pairwise-attention term, respectively.
The actual end-to-end speedup is smaller because encoder computation, projections, feed-forward layers, and decoding are not reduced by the same quadratic factor.

Ignoring model parameters shared by all settings, the sequence-dependent KV-cache memory changes from
$\mathcal{O}(L\bar Nd)$ to $\mathcal{O}(L\bar N_rd)$, while the allocator adds only $\mathcal{O}(Sd+U)$ memory.
Importantly, \method preserves the same $rN$ token count as any fixed-budget method at the same retained ratio.
Thus, its theoretical downstream cost is unchanged at a fixed budget; its purpose is to improve accuracy by placing that budget more effectively, with only linear or near-linear allocation overhead.

\section{Additional Details of the Motivation Studies}
\label{app:motivation_details}

This section supplements the two observation studies in Sec.~\ref{sec:motivation}.
We first document how the balanced modality-routing set was constructed and examine whether routing quality affects downstream question answering.
We then provide the complete protocols for the controlled video and audio budget-allocation diagnostics.

\subsection{Inter-Modal Budget Allocation Diagnostic}
\label{app:routing_details}

\paragraph{Dataset construction.}
GPT-5.5-xhigh inspects WorldSense questions and answer options and selects a diagnostic subset containing 200 strongly audio-oriented and 200 strongly video-oriented queries.
The selection depends only on the question and its candidates, rather than the ground-truth answer or a model prediction.
Mixed or ambiguous questions are excluded, and each retained query is annotated with the shortest phrase that most clearly indicates its required modality.
The prompt used for selection and annotation is reproduced below.
\clearpage
\begin{promptbox}
\small
\textbf{Prompt: balanced modality-routing subset construction.}

You need to annotate a balanced modality-routing subset from the WorldSense QA dataset. Each input sample contains a video identifier, task identifier, question, answer candidates, and, when available, its problem type and domain.

Select exactly 400 queries: 200 strongly audio-relevant queries and 200 strongly video-relevant queries. For every selected query, annotate the word in the query that most directly indicates the required modality.

An \emph{audio-relevant} query requires evidence such as speech, speaker identity, voice changes, music, rhythm, pitch, volume, sound events, sound counting, sound-source localization, or the presence of a sound. A \emph{video-relevant} query requires evidence such as objects, people, actions, appearance, clothing, color, on-screen text, diagrams, spatial relations, scene type, temporal visual order, or visible events.

Follow these rules:
\begin{enumerate}[leftmargin=*,nosep]
    \item Inspect all samples individually; do not randomly sample.
    \item Classify from the question and answer candidates, without using the ground-truth answer or a model prediction.
    \item Prefer high-confidence questions whose required modality is unambiguous, and exclude questions for which both modalities are necessary or either modality alone is plausible.
    \item If several cues occur, record the most diagnostic phrase and optionally record secondary cues. If the answer candidates reinforce the modality, mention this only in the reason.
    \item If more than 200 high-confidence samples are available for one modality, prioritize explicit modality cues, modality-specific candidates, high confidence, and diversity across tasks and domains.
\end{enumerate}

For every inspected sample, return its identifiers, question, candidates, candidate modality (audio, video, or mixed/ambiguous), confidence from 1 to 5, primary modality-relevant phrase, optional secondary phrases, and a brief reason. After inspection, return the final selected set with an \texttt{oracle\_modality} field for each item.

Before completing the task, verify that the subset contains exactly 400 samples, with exactly 200 audio and 200 video queries; every item has a nonempty relevant word; and no selected item is mixed or ambiguous.
\end{promptbox}

\paragraph{Representative examples.}
The following examples illustrate the resulting annotation criteria.

\noindent\textbf{Audio-oriented example.}
The query is ``How does the rhythm of the pipa music change at the beginning of the video?''
Its candidates are (A) from gentle to heavy, (B) from rapid to slow, (C) from heavy to gentle, and (D) from slow to rapid.
The primary relevant word is ``rhythm'', with ``pipa music'' and ``change'' as secondary cues.
Both the query and all candidates describe an acoustic temporal change, making the required modality unambiguous.

\noindent\textbf{Video-oriented example.}
The query is ``What is the man's appearance in the painting shown at the beginning of the video?''
Its candidates are (A) wearing 17th-century clothing with a hat, (B) wearing 17th-century clothing with long braids, (C) wearing 18th-century clothing with hair combed back, and (D) wearing 17th-century clothing with long hair and a beard.
The primary relevant phrase is ``appearance'', with ``painting shown'' as a secondary cue.
Answering the question requires inspection of visible attributes rather than acoustic evidence.

\paragraph{Effect on downstream accuracy.}
The routing experiment in Sec.~\ref{sec:similarity_analysis} evaluates whether a signal identifies the required modality, but routing is ultimately useful only if it improves the final answer.
We therefore compare three top-level policies on WorldSense with Qwen2.5-Omni-7B at a 25\% retained-token ratio.
All methods use the same \base intra-modal allocation and pruning backend.
The latter two methods differ only in the inter-modal routing signal and use the same query shift coefficient, $\lambda_q=0.05$.

\begin{table}[H]
\centering
\caption{Effect of the inter-modal routing signal on downstream WorldSense accuracy. All methods use the same intra-modal allocation and token-pruning policy. Best results are bolded.}
\label{tab:appendix_router_downstream}
\scriptsize
\setlength{\tabcolsep}{3.2pt}
\resizebox{\linewidth}{!}{%
\begin{tabular}{L{3.05cm}*{8}{C{0.90cm}}C{0.78cm}}
\toprule
\textbf{Inter-Modal Policy} & \textbf{Tech} & \textbf{Cult.} & \textbf{Daily} & \textbf{Film} & \textbf{Perf.} & \textbf{Games} & \textbf{Sports} & \textbf{Music} & \textbf{Avg.} \\
\midrule
\base & 47.1 & 46.9 & 43.6 & 39.6 & \textbf{40.4} & \textbf{39.9} & 40.0 & 45.1 & 43.2 \\
Embedding routing & 46.9 & 47.6 & 44.4 & \textbf{42.0} & 38.2 & 39.1 & 40.5 & \textbf{45.8} & 43.5 \\
\rowcolor{FullTokenBlue}
Skill-pool routing & \textbf{47.8} & \textbf{48.2} & \textbf{45.4} & 41.2 & 39.3 & \textbf{39.9} & \textbf{41.2} & 45.1 & \textbf{44.0} \\
\bottomrule
\end{tabular}%
}
\end{table}

Direct embedding routing yields only a small and category-dependent gain over the fixed prior, increasing the average from 43.2\% to 43.5\%.
Skill-pool routing improves the average to 44.0\% under the same internal allocator and pruning backend.
Together with its higher routing accuracy in Fig.~\ref{fig:router_similarity_comparison}, this result indicates that matching the query to modality-level task semantics provides a more useful inter-modal budget signal than matching it directly to sample-specific audio/video representations.

\subsection{Intra-Modal Budget Allocation Diagnostics}
\label{app:fine_grained_details}

The diagnostics in Sec.~\ref{sec:fine_grained_budget_motivation} are controlled, interpretable proxies for budget allocation and content pruning.
They are not intended to reproduce the exact tokenization or pruning pipeline of an OmniLLM.
Instead, visible image regions and short audio intervals serve as human-interpretable units, allowing the retained evidence under different budget policies to be inspected directly.

\paragraph{Video protocol and prompt.}
We sample 50 temporally coherent 2-second clips from WorldSense and extract four evenly spaced visual views from each clip.
GPT-5.5-xhigh generates a descriptive question and reference answer from the complete views, constructs five variants under the same 25\% total visible budget, answers from each retained view set, and grades the answer against the reference on the 0--5 scale used in Fig.~\ref{fig:fine_grained_budget_scores}.
Masked image regions are visual proxies for removed visual tokens.

\begin{promptbox}
\small
\textbf{Prompt: video budget-allocation study.}

Randomly select 50 samples from WorldSense. For each sample, choose a temporally coherent 2-second video clip and extract four evenly spaced visual views. Inspect the complete views and write one descriptive question about the scene, objects, actions, spatial details, or visible temporal change. Do not mention frame indices, masking, pruning, or budgets. Write a reference answer using only the complete visual content.

Construct five variants with the same total retained visual budget of 25\%:
\begin{enumerate}[leftmargin=*,nosep]
    \item \textbf{Uniform budget--Random pruning:} allocate 25\% to every view and randomly retain image regions within each view.
    \item \textbf{Weighted budget--Random pruning:} use the fixed view-level allocation $[0.70,0.15,0.10,0.05]$ and randomly retain regions within each view.
    \item \textbf{First-only:} spend the complete budget on the first view and remove the remaining views.
    \item \textbf{Uniform budget--Oracle pruning:} keep the budget uniform, but use the reference answer to retain the regions that best preserve answer-relevant evidence in each view.
    \item \textbf{Oracle budget--Oracle pruning:} use the reference answer to determine both the budget distribution and retained regions while keeping the same total budget.
\end{enumerate}

Cover removed regions with a mask. For each variant, answer the question using only its retained visible content; do not consult the complete views or reference answer while answering. Compare the generated answer with the reference and assign a score:
5, complete; 4, mostly complete; 3, partial but answerable; 2, limited but inferable; 1, weak fragments; 0, unanswerable.

Return the question, reference answer, five masked variants, five answers, five scores, and brief scoring rationales for every sample. Finally, report each policy's average over the 50 samples.
\end{promptbox}

\paragraph{Audio protocol and prompt.}
We sample 50 WorldSense audio files and extract one 12-second clip from each sample.
Each clip is divided into four consecutive 3-second segments and then into 24 chunks of 0.5 seconds.
Every policy retains exactly 12 chunks (50\%), and each removed chunk is replaced with silence so that all variants remain 12 seconds long.
Gemini-3.1-Pro-Preview first generates the question, reference answer, and answer-aware chunk priorities from the complete audio; it then answers and scores each variant using only the corresponding pruned audio.
\clearpage
\begin{promptbox}
\small
\textbf{Prompt: audio budget-allocation study.}

Randomly select 50 samples from the WorldSense audio set. For each sample, extract a 12-second clip, divide it into four 3-second segments, and divide each segment into six 0.5-second chunks, giving 24 chunks in total. Listen to the complete audio and write one descriptive question about its overall audible content, such as speech, sound events, background music, speaker changes, or temporal progression. Avoid questions that can be answered from a single isolated keyword. Write a complete reference answer from the original audio.

Assign an importance ranking or score to every 0.5-second chunk according to its usefulness for answering the question and recovering the reference answer. Use these scores only to determine retained chunks. Construct five variants with the same total budget of 12 retained chunks:
\begin{enumerate}[leftmargin=*,nosep]
    \item \textbf{Uniform budget--Random pruning:} randomly retain three of the six chunks in every segment.
    \item \textbf{Weighted budget--Random pruning:} retain $[5,4,2,1]$ chunks in the four segments and select chunks randomly within each segment.
    \item \textbf{First-only:} retain all chunks from the first two segments and remove all chunks from the last two segments.
    \item \textbf{Uniform budget--Oracle pruning:} retain three chunks per segment, selecting the most important chunks according to the answer-aware ranking.
    \item \textbf{Oracle budget--Oracle pruning:} distribute the 12-chunk budget freely across segments and retain the globally most answer-relevant chunks.
\end{enumerate}

Replace removed chunks with silence. Answer the question using only the current pruned audio, without access to the original clip, reference answer, or chunk-importance scores. Compare each answer with the reference and assign a score:
5, complete; 4, mostly complete; 3, partial but answerable; 2, limited but inferable; 1, weak fragments; 0, unanswerable.

Return the question, reference answer, chunk priorities, five pruned audio clips, five answers, five scores, and brief scoring rationales for every sample. Finally, report each policy's average over the 50 samples.
\end{promptbox}

\paragraph{Controls and interpretation.}
The five policies separate the effects of budget placement and retained-content selection.
Uniform--Random fixes both a uniform budget and an uninformed selector; Weighted--Random changes only the budget distribution; and First-only represents an extreme front-loaded allocation.
The two Oracle variants are answer-aware only when constructing retained content: the answering stage never observes the reference answer or original uncompressed input.
Thus, Uniform--Oracle estimates the benefit of stronger pruning under an unchanged uniform budget, whereas Oracle--Oracle estimates the benefit of jointly improving budget placement and content selection.

\begin{table}[H]
\centering
\caption{Average answer scores in the controlled video and audio diagnostics. All policies within each modality use the same total retained budget. ``Oracle'' denotes answer-aware retention construction; answering uses only the resulting pruned input.}
\label{tab:appendix_budget_diagnostic}
\small
\setlength{\tabcolsep}{10pt}
\begin{tabular}{lcc}
\toprule
\textbf{Budget--Pruning Policy} & \textbf{Video} & \textbf{Audio} \\
\midrule
Uniform--Random & 2.03 & 2.92 \\
Weighted--Random & 2.50 & 3.08 \\
First-only & 3.17 & 3.12 \\
Uniform--Oracle & 3.25 & 3.14 \\
Oracle--Oracle & \textbf{3.55} & \textbf{3.40} \\
\bottomrule
\end{tabular}
\end{table}

The two modalities exhibit the same qualitative pattern.
Changing only the budget distribution improves over Uniform--Random even when content is selected randomly.
Conversely, Uniform--Oracle remains close to First-only and below Oracle--Oracle despite using answer-aware pruning, showing that strong local selection cannot recover evidence from a unit whose assigned budget is insufficient.
These diagnostics isolate the motivation for intra-modal allocation: under a fixed total budget, informative temporal regions require more capacity, whereas redundant regions can be compressed more aggressively.

\section{Modality Skill-Pool Details}
\label{app:skill_pool}

This section supplements the query-aware inter-modal allocator in Sec.~\ref{sec:skill_budget}.
We first describe how modality-specific task types and their lexical cues are extracted from WorldSense, and then present representative entries from the resulting audio and video skill pools.

\subsection{Skill-Pool Construction}
\label{app:skill_construction}

The skill pools are constructed offline with GPT-5.5-xhigh.
WorldSense questions and answer candidates are used to identify representative task types that predominantly require either acoustic or visual evidence.
For each task type, the model extracts compact modality-indicative keywords and expands them with semantically related words and short phrases.
The construction prompt is shown below.

\begin{promptbox}
\small
\textbf{Prompt: construction of audio and video modality skill pools.}

You are constructing two lexical skill pools for query-aware modality routing in an omni-modal language model. The input is the WorldSense QA dataset, including each question, its answer candidates, and, when available, task and domain metadata. The pools will be embedded with the model's text embedding layer and compared with an embedded query using cosine similarity. They must therefore describe the modality and task required by a query, rather than the sample-specific content of one video.

Complete the following steps:
\begin{enumerate}[leftmargin=*,nosep]
    \item Inspect the questions and answer candidates and identify representative task types that strongly require audio or video evidence. Exclude mixed or ambiguous task types for which both modalities are indispensable.
    \item Organize audio tasks around speech and dialogue, speaker and voice, music and rhythm, sound-event recognition, sound source, event counting, acoustic changes, loudness, and silence. Organize video tasks around object and scene recognition, appearance and attributes, spatial relations, on-screen text and diagrams, action and motion, temporal visual change, interaction, and visible event reasoning.
    \item For each task type, select the smallest keywords or short query phrases that indicate the required evidence. Prefer general task cues that recur across samples rather than named entities, answer-specific values, or descriptions tied to one clip.
    \item Expand each cue with useful synonyms, lexical variants, and short semantic paraphrases. For example, a speech task may contribute ``voice'', ``spoken words'', and ``who is speaking'', while an OCR task may contribute ``text'', ``caption'', and ``read on-screen text''.
    \item Keep entries concise enough to represent a single skill. Use lowercase text, preserve meaningful multiword expressions, and avoid generic function words or the uninformative word ``video'' by itself.
    \item Audit every entry for cross-modal ambiguity. A visible instrument does not make a sound-related skill visual, and the presence of a speaker does not make a speech-understanding skill visual. Retain a cue only when it reliably indicates the evidence needed to answer the query.
    \item Cover diverse task families rather than overpopulating one category. The two modality pools need not have identical sizes, but each should represent the task diversity observed in WorldSense.
\end{enumerate}

Produce four lists. \texttt{audio\_skills} and \texttt{video\_skills} contain compact words and lexical expressions;\\ \texttt{audio\_skill\_phrases} and \texttt{video\_skill\_phrases} contain short task-level descriptions. Target approximately 60--120 compact entries and 10--30 task phrases per modality. Return a JSON object with the fields \texttt{version}, \texttt{description}, \texttt{audio\_skills}, \texttt{video\_skills}, \texttt{audio\_skill\_phrases}, and \texttt{video\_skill\_phrases}.

Before returning the JSON, verify that every item is nonempty, assigned to the appropriate modality, concise enough for embedding-based retrieval, and collectively covers the major audio- and video-oriented task types.
\end{promptbox}

The resulting bank is stored in \texttt{query\_modality\_skill\_pool.json} and contains 73 compact audio entries and 105 compact video entries, together with 17 task-level phrases for each modality.
At inference time, the compact and phrase lists are concatenated within each modality, producing an audio bank of 90 entries and a video bank of 122 entries.
Each entry is encoded once using the OmniLLM thinker input embedding layer, mean-pooled, $\ell_2$-normalized, and cached.
Consequently, GPT-5.5-xhigh is used only for offline pool construction and introduces no online routing cost.

\subsection{Representative Skill-Pool Entries}
\label{app:skill_examples}

The following boxes provide a partial view of the skill pools used in our experiments.
Entries are grouped only for readability; the implementation uses one flat bank per modality and applies Top-$k$ retrieval over all compact skills and task-level phrases jointly.

\begin{skillbox}
\small
\textbf{Audio skill pool (selected entries)}
\smallskip

\renewcommand{\arraystretch}{1.15}
\noindent\begin{tabularx}{\linewidth}{@{}L{3.0cm}X@{}}
\textbf{Task family} & \textbf{Representative skills} \\
\midrule
General listening & \emph{audio, acoustic, sound, heard, hear, listen, listening} \\
Speech and speaker & \emph{voice, speech, spoken words, speaker, dialogue, narration, who is speaking, what is said} \\
Music and rhythm & \emph{music, song, singing, instrument, guitar, piano, drum, melody, rhythm, beat, background music} \\
Sound event and source & \emph{applause, cheering, noise, sound effect, engine sound, footsteps, whistle, alarm, bell, audio source, sound source} \\
Temporal and acoustic properties & \emph{count sounds, repeated sound, audio change, sound transition, volume, loud, quiet, silence} \\
Task-level phrases & \emph{answer by listening, audio evidence, sound event recognition, speech and dialogue, speaker and voice, count audio events, audio changes over time, identify the sound, listen to the soundtrack} \\
\end{tabularx}
\end{skillbox}

\begin{skillbox}
\small
\textbf{Video skill pool (selected entries)}
\smallskip

\renewcommand{\arraystretch}{1.15}
\noindent\begin{tabularx}{\linewidth}{@{}L{3.0cm}X@{}}
\textbf{Task family} & \textbf{Representative skills} \\
\midrule
Scene and object & \emph{visual, visible, shown, appearance, scene, setting, location, object, people, animal} \\
Attribute and spatial relation & \emph{clothing, wearing, color, shape, size, position, left, right, above, behind, spatial relation} \\
On-screen text and diagram & \emph{text, written, caption, subtitle, sign, label, diagram, chart, screen text, visible number} \\
Action and motion & \emph{action, activity, motion, movement, gesture, pose, what someone does} \\
Temporal and interaction reasoning & \emph{event order, sequence, before and after, state change, reaction, facial expression, interaction, cause, effect, outcome, prediction} \\
Task-level phrases & \emph{answer by looking, visual evidence, visible objects, visual appearance, scene understanding, spatial layout, read on-screen text, visible action, motion in video, temporal visual event, track visual changes} \\
\end{tabularx}
\end{skillbox}

Short keywords and their lexical variants cover diverse query expressions, while the longer phrases encode task intent more explicitly.
Because both forms share the same embedding bank, Top-$k$ matching can retrieve whichever granularity best aligns with the query without assigning hand-crafted weights to individual task families.

\end{document}